\documentclass[runningheads]{llncs}

\usepackage{eccv}

\usepackage{eccvabbrv}

\usepackage{graphicx}
\usepackage{booktabs}
\usepackage{subcaption}
\usepackage{caption}
\usepackage{wrapfig}
\usepackage{svg}
\usepackage{lipsum}
\usepackage[accsupp]{axessibility}  
\usepackage[table]{xcolor}
\usepackage{colortbl}
\definecolor{best}{RGB}{231,209,255}     
\definecolor{second}{RGB}{243,229,255}   
\definecolor{third}{RGB}{250,242,255}    %

\usepackage{pifont}

\newcommand{\cmark}{\textcolor{green!60!black}{\ding{51}}} 
\newcommand{\xmark}{\textcolor{red!70!black}{\ding{55}}}   

\usepackage{arydshln}

\usepackage[pagebackref,breaklinks,colorlinks,citecolor=eccvblue]{hyperref}

\usepackage{orcidlink}
\usepackage{multirow}

\begin{document}

\title{VGGT-World: Transforming VGGT into an Autoregressive Geometry World Model} 

\titlerunning{VGGT-World}

\author{
Xiangyu Sun\inst{1} \and
Shijie Wang\inst{1} \and
Fengyi Zhang\inst{1} \and
Lin Liu\inst{2} \and
Caiyan Jia\inst{2} \and
Ziying Song\inst{2} \and
Zi Huang\inst{1} \and
Yadan Luo\inst{1}\thanks{Corresponding author.}
}

\authorrunning{X. Sun et al.}

\institute{
UQMM Lab, The University of Queensland \\
\and
Beijing Jiaotong University \\
}

\maketitle

\begin{abstract}
  World models that forecast scene evolution by generating future video frames devote the bulk of their capacity to photometric details, yet the resulting predictions often remain geometrically inconsistent. We present VGGT-World, a \emph{geometry world model} that side-steps video generation entirely and instead forecasts the temporal evolution of frozen geometry-foundation-model (GFM) features. Concretely, we repurpose the latent tokens of a frozen VGGT as the world state and train a lightweight temporal flow transformer to autoregressively predict their future trajectory. Two technical challenges arise in this high-dimensional ($d{=}1024$) feature space: (i)~standard velocity-prediction flow matching collapses, and (ii)~autoregressive rollout suffers from compounding exposure bias. We address the first with a clean-target ($z$-prediction) parameterization that yields a substantially higher signal-to-noise ratio, and the second with a two-stage latent flow-forcing curriculum that progressively conditions the model on its own partially denoised rollouts. Experiments on KITTI, Cityscapes, and TartanAir demonstrate that VGGT-World significantly outperforms the strongest baselines in depth forecasting while running 3.6-5$\times$ faster with only 0.43B trainable parameters, establishing frozen GFM features as an effective and efficient predictive state for 3D world modeling.
\end{abstract}

\section{Introduction}
\label{sec:intro}
Recent 3D geometry foundation models (GFMs) \cite{wang2025vggt,wang2025pi,DBLP:journals/corr/abs-2509-13414,ma2026metricanything,DBLP:conf/cvpr/Wang0CCR24,DBLP:conf/eccv/LeroyCR24,DBLP:conf/iclr/Hong0GBZLLSB024}, such as VGGT \cite{wang2025vggt} have demonstrated that large transformers can recover holistic 3D scene structure from monocular RGB observations, including camera poses, depth maps, and point clouds. These models provide embodied agents with a strong geometric representation \cite{li2025spatial} of the observed world, especially benefiting tasks like collision-free navigation \cite{DBLP:journals/corr/abs-2507-07982,DBLP:journals/corr/abs-2511-22264,DBLP:journals/corr/abs-2405-20323} and manipulation \cite{DBLP:conf/cvpr/WangMZXLLCZCXLL24,DBLP:journals/corr/abs-2412-04380}. However, they remain fundamentally \textit{observational}: they infer geometry from \textit{available views}, but do not predict how that geometry will evolve over time. For planning and control in dynamic environments, an agent needs not only perception, but also \textbf{foresight}: the ability to anticipate future scene states that remain consistent with underlying geometric and physical dynamics.

\begin{figure}[t]
    \centering
    \includegraphics[width=1\linewidth]{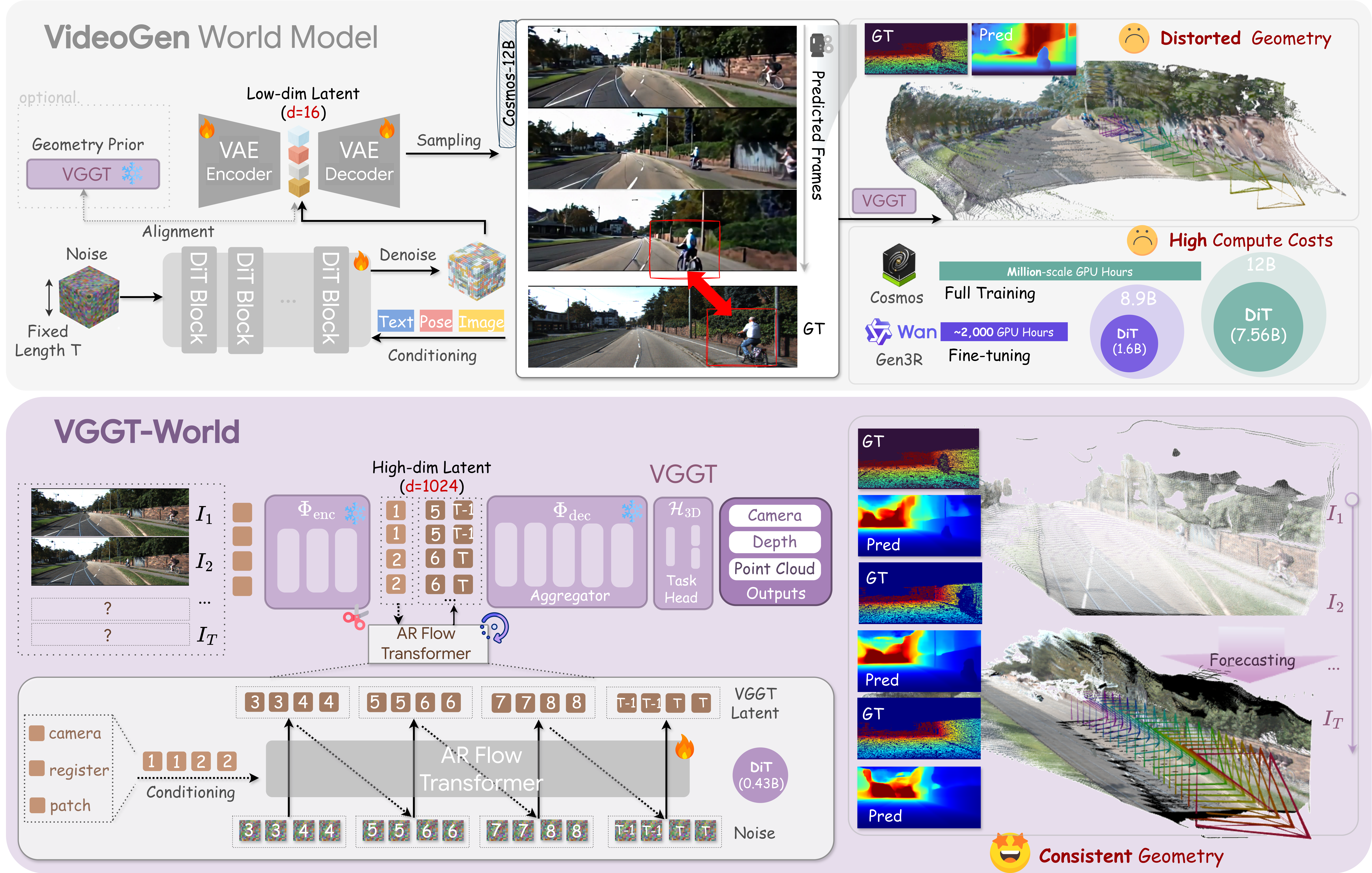}
    \caption{\textbf{From video world models to geometry world models.} Video world models predict future RGB in VAE latent space, coupling scene dynamics with appearance reconstruction. As a result, decoded predictions can remain geometrically invalid, with broken scene layout and mislocalized actors despite plausible video appearance. VGGT-World instead uses frozen geometry-foundation features as the latent state and models their temporal evolution directly, yielding a VAE-free, lightweight, and geometry-consistent alternative for future 3D forecasting.
}
\label{fig:overview}\vspace{-2ex}
\end{figure}
A natural way to equip agents with such foresight is to cast video generation models as \textit{world models}. Representative systems such as Wan \cite{wan2025wan}, MovieGen \cite{polyak2024movie}, and Cosmos \cite{DBLP:journals/corr/abs-2511-00062} typically compress videos into VAE latents \cite{DBLP:journals/corr/KingmaW13}, and train diffusion or autoregressive predictors to reconstruct future RGB frames. While visually impressive, this video-centric paradigm is \textit{poorly} suited to geometry forecasting: its objective prioritizes photometric details such as texture, lighting, and appearance changes, which is loosely coupled to 3D structure. As illustrated in Fig. \ref{fig:overview}, decoded future videos can still produce severely \textit{distorted} geometry: the recovered structure may collapse into fragmented sheet-like surfaces, and dynamic objects can be mislocalized (\eg, a cyclist may be shifted in front of the ego vehicle). Such failures are particularly concerning because they can corrupt scene understanding and potentially mislead downstream decisions. Even geometry-aware video generation variants that inject geometric priors, such as Gen3R~\cite{DBLP:journals/corr/abs-2601-04090}, GeoWorld~\cite{DBLP:journals/corr/abs-2511-23191}, GeometryForcing~\cite{DBLP:journals/corr/abs-2507-07982}, and GeoVideo~\cite{bai2025geovideo}, remain centered on future RGB reconstruction rather than direct prediction of 3D geometric evolution. Beyond this representational mismatch, the paradigm is computationally \textit{expensive}: Cosmos reportedly requires tens of millions of GPU-hours, while even Gen3R requires multi-day training on 24 H20 GPUs to finetune its VAE, making such pipelines impractical for most academic labs.

These limitations suggest that the key issue is not the forecasting mechanism itself, but how the world state is represented. Instead of learning a video VAE space, we argue that GFM features that already encodes multi-view geometry and metric structure are a more suitable state space for predictive 3D modeling. We term this perspective \textbf{geometry world modeling}: world modeling is formulated as geometry-state evolution rather than future video construction. 
Building on this idea, we introduce \textbf{VGGT-World} (Fig.~\ref{fig:overview}), which repurposes frozen VGGT features as the world state and learns only their temporal evolution. Concretely, we extract tokens from observed frames as conditions and train a temporal flow transformer to autoregressively predict future camera and patch tokens in a chunk-wise manner. The predicted tokens are then decoded using a frozen VGGT decoder and heads into 3D geometric outputs such as depth, point maps, and camera-motion cues, enabling arbitrary-horizon geometry forecasting without video VAE training or heavy video-backbone adaptation.

However, learning a geometry world model directly in frozen GFM feature space is \textit{non-trivial}. Compared with standard video VAE latents ($d=16$), VGGT states are weakly compressed and remain substantially higher-dimensional ($d=1024$), making direct flow-based prediction difficult to optimize. In addition, autoregressive rollout introduces exposure bias, since each predicted chunk must condition on previously generated geometry states that may gradually drift away from the ground truth. To address these challenges, we adopt a $z$-prediction parameterization for stable learning in high-dimensional geometry space, together with a simple two-stage flow-forcing curriculum strategy that bridges teacher-forced training and inference-time rollout.
These designs enable stable and effective learning of geometry dynamics in high-dimensional GFM feature space.

We evaluate VGGT-World on multiple datasets, including KITTI \cite{DBLP:journals/ijrr/GeigerLSU13}, Citysc- apes \cite{DBLP:conf/cvpr/CordtsORREBFRS16}, and TartanAir \cite{DBLP:conf/iros/WangZWHQWHKS20}, across a range of geometry forecasting tasks such as future depth prediction, 3D point-map forecasting, camera trajectory preservation, and efficiency analysis. Across KITTI and Cityscapes, VGGT-World consistently surpasses prior world models, reducing mean AbsRel by up to 21\% and 32\% over the strongest prior baselines while improving $\delta_1$ across both short- and mid-term horizons. It is also markedly more efficient, being 3.6-5$\times$ faster than Cosmos (12B) and Gen3R (9B) with only 0.43B trainable parameters. Together, these results establish frozen geometry foundation features as an effective and efficient predictive state for 3D geometry world modeling.

\section{Related Work}
\textbf{Geometry Foundation Model.} 3D scene understanding has recently shifted from iterative geometric optimization pipelines, such as Structure-from-Motion (SfM)~\cite{DBLP:journals/cacm/AgarwalFSSCSS11,DBLP:conf/eccv/FrahmGGJRWJDCL10,DBLP:conf/cvpr/SchonbergerF16,DBLP:journals/tog/SnavelySS06,DBLP:conf/3dim/Wu13} and Bundle Adjustment~\cite{DBLP:journals/toms/LourakisA09,DBLP:conf/iccvw/TriggsMHF99}, toward end-to-end learning. Early feed-forward multi-view reconstruction models such as DUSt3R \cite{DBLP:conf/cvpr/Wang0CCR24} demonstrate that camera poses, depth, and point clouds can be inferred directly from unposed image collections in a single forward pass. Building on this paradigm, recent geometry foundation models (GFMs), including VGGT \cite{wang2025vggt}, $\pi^3$ \cite{wang2025pi}, and MapAnything \cite{DBLP:journals/corr/abs-2509-13414,ma2026metricanything} scale geometry learning toward modality-general, representations with flexible input–output geometry prediction. However, these models remain limited to observational inputs: They excel at analyzing what is from available views but lack the ability to model temporal dynamics for forecasting future states.

\noindent\textbf{Video Generation.} Predictive modeling of scene evolution has largely been studied in the video generation community \cite{VAR,Video_diffusion_model}. Frontier systems \cite{NOVA,CausVid} such as SVD~\cite{SVD}, Sora~\cite{liu2024sora}, Cosmos~\cite{DBLP:journals/corr/abs-2511-00062}, Wan2.1~\cite{wan2025wan} learn powerful spatiotemporal priors via bidirectional denoising over full temporal windows, achieving high perceptual realism and motion coherence. Despite recent progress, these models operate primarily in pixel~\cite{videogpt} or compressed video latent spaces~\cite{CausVid} and thus lack explicit 3D structure, which limits geometric consistency over long horizons. To address this, a growing line of works introduces geometric priors into video generation. Representative models such as Geometry Forcing \cite{DBLP:journals/corr/abs-2507-07982}, GeoWorld \cite{DBLP:journals/corr/abs-2511-23191}, and Gen3R \cite{DBLP:journals/corr/abs-2601-04090} align diffusion generation with multi-view or reconstruction-based 3D representations, improving view consistency and structural stability. This trend reflects an emerging consensus that predictive world modeling requires geometry-aware representations rather than purely appearance-based video latents. Nevertheless, most generative video models still rely on VAE-compressed latents that entangle geometry and appearance, limiting causal 3D dynamics.

\noindent\textbf{Latent World Model.} Recent latent world models \cite{karypidis2025dinoforesight,DBLP:journals/corr/abs-2507-19468,law,world4drive,jia2026driveworld} increasingly replace VAE latents \cite{DBLP:journals/corr/abs-2510-15301} with the feature spaces of strong vision encoders such as DINOv2 \cite{DBLP:conf/iccv/CaronTMJMBJ21}, DINOv3 \cite{DBLP:journals/corr/abs-2508-10104} and World4Drive~\cite{world4drive}. DINO-Foresight \cite{karypidis2025dinoforesight} forecasts frozen DINO feature tokens with a masked feature transformer to enable future depth or segmentation prediction without pixels, while DINO-World \cite{DBLP:journals/corr/abs-2507-19468} learns video dynamics directly in DINOv2 feature space, showing that DINO features alone can serve as a generalist world-model latent. LAW~\cite{law} and WOTE~\cite{WOTE} performs future prediction directly in the BEV latent space. Together, these results suggest that predictive performance depends more on the choice of latent representation than on pixel-level reconstruction. However, DINO representations are primarily optimized for semantic invariance rather than explicit 3D geometry. In contrast, GFMs encode multi-view geometry, spatial consistency, and metric structure, making them inherently better suited as latent backbones for predictive 3D world modeling. We therefore advocate autoregressive world modeling directly in geometry foundation feature spaces.

\begin{figure}[t]
\centering
\begin{minipage}[b]{0.39\linewidth}
\centering
{\captionsetup[subfloat]{labelformat=empty,font=footnotesize,skip=1pt}
\subfloat[Latent SNR: $z$-pred vs $v$-pred]{%
  \includegraphics[width=.97\linewidth]{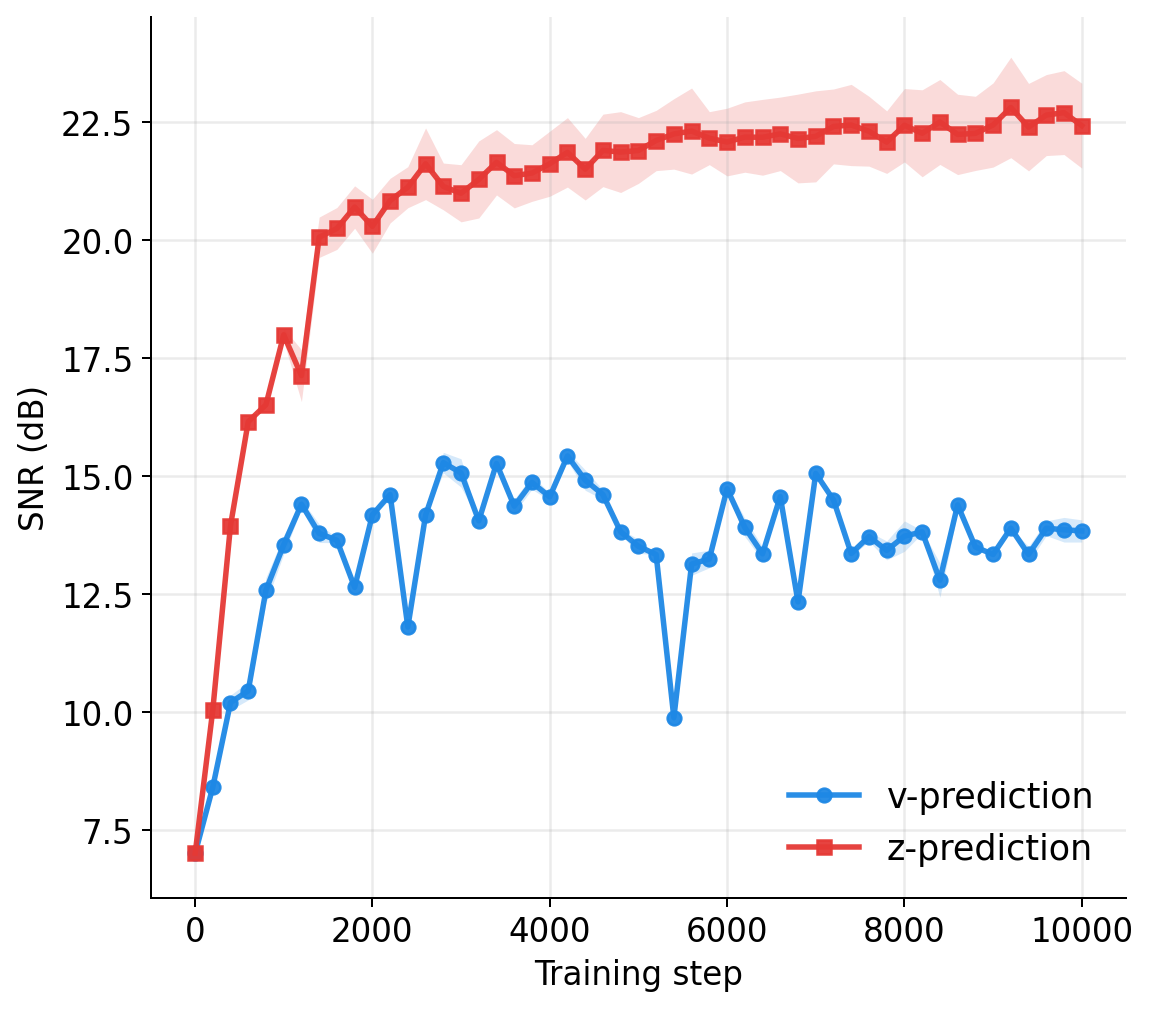}
}}
\end{minipage}
\hfill
\begin{minipage}[b]{0.6\linewidth}
\centering

\newcommand{\vsep}{\hspace{3pt}\textcolor{gray!30}{\vrule width 0.6pt}\hspace{6pt}}


{\captionsetup[subfloat]{labelformat=empty,font=footnotesize,skip=1pt}
\raisebox{0pt}[0pt][0pt]{%
\makebox[0.28\linewidth][c]{\scriptsize}%
\hfill
\makebox[0.25\linewidth][c]{\scriptsize \quad 200 iter $\longrightarrow$}%
\makebox[0.25\linewidth][c]{\scriptsize  \quad 1,000 iter $\longrightarrow$}%
\makebox[0.24\linewidth][c]{\scriptsize 4,000 iter}%
}
\subfloat[VGGT]{\includegraphics[width=0.23\linewidth]{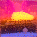}}
\hfill\vsep
\subfloat[]{\includegraphics[width=0.23\linewidth]{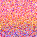}}
\hfill
\subfloat[$v$-pred train]{\includegraphics[width=0.23\linewidth]{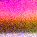}}
\hfill
\subfloat[\,]{\includegraphics[width=0.23\linewidth]{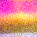}}


\setlength{\fboxsep}{0pt}\setlength{\fboxrule}{0.3pt}
\subfloat[Original]{\fbox{\includegraphics[width=0.23\linewidth]{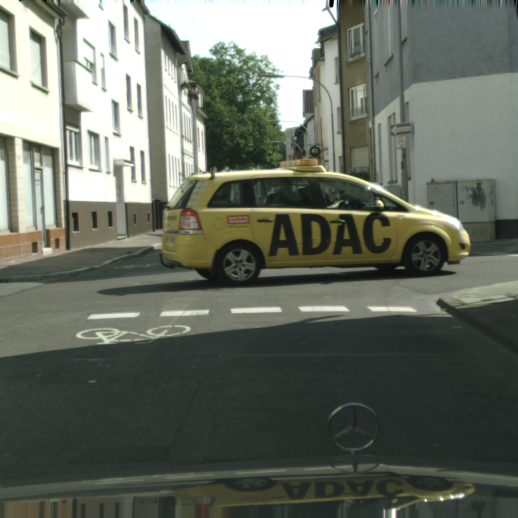}}}
\hfill\vsep
\subfloat[\,]{\includegraphics[width=0.23\linewidth]{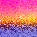}}
\hfill
\subfloat[$z$-pred train]{\includegraphics[width=0.23\linewidth]{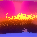}}
\hfill
\subfloat[\,]{\includegraphics[width=0.23\linewidth]{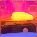}}

} 

\end{minipage}

\caption{\textbf{Comparison of $v$-prediction and $z$-prediction in frozen VGGT latent space.} Left: $z$-prediction consistently achieves substantially higher signal-to-noise ratio than $v$-prediction in the layer-4 latent space. Middle: A target VGGT feature map and its corresponding RGB frame. Right: PCA visualizations of predicted latents during training. $v$-prediction remains noisy even after prolonged training, whereas $z$-prediction progressively recovers structured latent patterns aligned with the target VGGT feature.}\label{fig:preliminary}
\end{figure}


\section{Our Approach} 
\noindent\textbf{Overview.} As depicted in Fig. \ref{fig:overview}, we formulate world modeling as predicting the temporal evolution of \textit{geometry states} in the latent space of a frozen geometry foundation model, rather than reconstructing future RGB frames. VGGT-World separates spatial representation learning from temporal dynamics modeling: a frozen VGGT encoder extracts deterministic geometry states from observed frames (\S\ref{sec:state}), a lightweight temporal flow transformer predicts their future evolution (\S\ref{sec:flow}), and the frozen VGGT decoder together with the original 3D task heads converts the predicted trajectory into geometric outputs. This design avoids training a video VAE or adapting a heavy video-generation backbone while preserving compatibility with the pretrained geometry decoder. To mitigate train-test mismatch during recursive rollout, we further introduce a two-stage latent flow-forcing curriculum that gradually exposes the model to its own generated histories (\S\ref{sec:training}).

\subsection{Geometry World Modeling via Autoregressive Flow}\label{sec:state}
To model the temporal evolution of scene geometry, we first extract deterministic geometry states from observed history frames ($k$ views), then iteratively forecast their future evolution using an autoregressive flow model, and finally decode the assembled trajectory into full 3D representations.

\noindent\textbf{Geometry State Representation.} Formally, given a video sequence of $T$ frames $\mathbf{I}_{1:T}=(I_1, I_2, \ldots, I_T)$, our goal is to model its temporal geometric evolution. We reuse the representation space of a frozen pretrained VGGT model $\Phi(\cdot)$, which consists of $L=48$ (24 self-attention and 24 global-attention) transformer blocks. The frozen VGGT 3D heads $\mathcal{H}_{\operatorname{3D}}(\cdot)$ consume a multi-scale feature hierarchy from layers $\{4, 11,17, 23\}$ in both block types. Predicting deeper layers (\eg $L_{\text{17}}$) would discard the required early layer (\eg $L_{\{\text{4},11\}}$) representation, which cannot be recovered by the forward network. We therefore use the $L_{\text{4}}$) output as the predictive state, as it is the earliest decoder-compatible representation. We partition $\Phi$ into a deterministic geometry encoder $\Phi_{\operatorname{enc}}$ (layer $0$-$4$) and a feature propagator $\Phi_{\operatorname{dec}}$ (layers $5$-$47$) to produce tokenized clean geometry states:
\begin{equation}
    \mathbf{z}_t = \Phi_{\text{enc}}(I_t) \in \mathbb{R}^{N \times d},
\end{equation}
where $t$ is the discrete temporal index and $d=1024$ is the embedding dimension. The sequence length $N = 5 + N_p$ comprises $N_p$ spatial patch tokens and $5$ special tokens (\eg, camera and register tokens).  

\noindent\textbf{Chunk-wise Autoregressive Rollout.} Rather than predicting one future frame at a time, we forecast geometry states in temporal chunks of length $m$. Given an observed history of $k$ states $\mathbf{z}_{1:k}$, our goal is to forecast the geometric evolution up to a future horizon $T$. Assuming the future spans $S$ chunks such that $T = k + S \cdot m$, we denote the $i$-th target chunk as $\mathbf{Z}_i = \mathbf{z}_{k+(i-1)m+1 : k+im}$. Let $\mathbf{c}_i$ denote the context used to predict chunk $\mathbf{Z}_i$ formed by the most recent $k$ geometry states in the trajectory. The rollout distribution factorizes as:
\begin{equation}\label{eq:AR}
    p(\mathbf{z}_{k+1:T} \mid \mathbf{z}_{1:k}) = \prod_{i=1}^{S} p(\mathbf{Z}_i \mid \mathbf{c}_i).
\end{equation}
To maintain a constant context size and bound memory complexity, we do not accumulate a growing history of context during inference. Instead, the model acts as a sliding window: it predicts a full chunk of $m$ states $\hat{\mathbf{Z}}_i$, and then extracts the latest $k$ frames from the recently decoded sequence to serve as the fixed-size condition $\mathbf{c}_{i+1}=\hat{\mathbf{Z}}_i$ for the subsequent step. This allows the model to iteratively roll out chunk-by-chunk to construct the full sequence without context explosion.

\noindent\textbf{3D Joint Decoding with Forecasted Horizon.} A unique property of VGGT-World is that forecasting the unseen future \textit{improves} the geometric estimation of the seen present. Inferring holistic 3D geometry from a highly sparse context (\eg, $k=2$) is inherently ambiguous. Our framework naturally resolves this by leveraging the forecasted future as a temporal regularizer. After the rollout produces the future chunks $\hat{\mathbf{Z}}_{1:S}$, we aggregate them with the initial sparse context $\mathbf{z}_{1:k}$ to form a latent trajectory of the full scene:
\begin{equation}
    \mathbf{z}_{1:T}^{\text{full}} = [\mathbf{z}_{1:k} ; \hat{\mathbf{Z}}_{1:S}].
\end{equation}
This assembled sequence is then propagated jointly through the remaining frozen blocks $\Phi_{\operatorname{dec}}$, and the original VGGT 3D heads decode the final outputs (\eg, point clouds $\mathcal{P}_{1:T}$, depth maps $\mathcal{D}_{1:T}$):
\begin{equation}
    \mathcal{P}_{1:T}, \mathcal{D}_{1:T} = \mathcal{H}_{\text{3D}}(\Phi_{\operatorname{dec}}(\mathbf{z}_{1:T}^{\text{full}})).
\end{equation} 
As a result, the geometric precision at the observed step $k$ is significantly enhanced precisely because the model is forced to contextualize it within a consistent, long-horizon world model rollout.
\subsection{Chunk Transition in High-Dimensional Geometry Space} \label{sec:flow}
To model the conditional transition of the next chunk $p(\mathbf{Z}_i \mid \mathbf{c}_i)$ in Eq. \eqref{eq:AR}, we adopt continuous-time flow matching \cite{DBLP:conf/iclr/LiuG023} in the geometry latent space. Standard flow matching \cite{DBLP:conf/iclr/LiuG023} defines a probability path $\mathbf{Z}_{i, \tau} = (1-\tau)\mathbf{Z}_i + \tau \epsilon$ governed by the ODE:
\begin{equation}
    \frac{d \mathbf{Z}_{i, \tau}}{d\tau} = v_\theta(\mathbf{Z}_{i, \tau}, \tau, \mathbf{c}_i),
\end{equation} where $v_\theta$ is typically trained to predict the vector field velocity ($\mathbf{v}$-prediction).

\noindent\textbf{The Curse of Dimensionality.} However, we empirically found that standard $\mathbf{v}$-prediction \textbf{\textit{fails}} in our setting (Fig. \ref{fig:preliminary}): even after 4k training iterations, the denoised predictions remain noisy and weakly structured. We attribute this failure to the \textit{high} dimensionality of our target space: unlike video VAEs operating in heavily compressed latents (\eg, 16–32 channels), we model dynamics directly in the 1024-dimensional VGGT feature spacewhere valid geometry states occupy only a structured subset of the ambient space. Isotropic Gaussian corruption therefore, mainly perturbs states along off-manifold directions, rendering velocity estimation ill-conditioned. In this regime, small velocity errors quickly drive trajectories away from the true geometry manifold and are further amplified by $\Phi_{\text{dec}}$ and $\mathcal{H}_{\text{3D}}$. While a common remedy is to learn an additional dimensionality reduction module \cite{karypidis2025dinoforesight}, this would break alignment with the frozen VGGT heads and destroy their zero-shot geometric priors.

\noindent\textbf{$z$-prediction Flow.} We therefore directly predict the clean target latent $\mathbf{Z}_i$ instead of its velocity. In spirit, this is related to recent findings \cite{DBLP:journals/corr/abs-2511-13720,baade2026latent} that direct clean-target prediction can provide stronger denoising signals than velocity prediction; however, unlike pixel-space $x$-prediction, our model predicts the clean geometry latent state. Fig. \ref{fig:preliminary} confirms this choice empirically: $\mathbf{z}$-prediction yields higher training SNR and progressively recovers structured target latents.

\noindent\textbf{Implementation.} Consequently, we explicitly parameterize our network $F_\theta$ to predict the clean latents of the target chunk directly ($z$-prediction). To enforce the separation of temporal reasoning and spatial refinement, we decouple $F_\theta$ into a composition of two distinct sub-networks: a dual-stream causal processor $F_{\text{dual}}$ and a single-stream spatial denoiser $F_{\text{single}}$:
\begin{equation}
    \mathbf{Z}_{i} \approx F_\theta(\mathbf{Z}_{i,\tau}, \tau, \mathbf{c}_i) = F_{\text{single}}\Big( F_{\text{dual}}(\mathbf{Z}_{i,\tau}, \tau, \mathbf{c}_i), \; \tau \Big).
\end{equation}
This architectural design aligns with very recent analyses \cite{bai2026causality} showing that causality in video diffusers is separable from spatial denoising: temporal attention naturally dominates \textit{early layers} to establish historical context, while \textit{deeper layers} become highly diagonal, focusing almost entirely on intra-frame refinement. Specifically, continuous flow time $\tau$ is injected into all blocks via Adaptive Layer Normalization (adaLN). For early dual-stream blocks $l \in [1, L_{\text{d}}]$, the noisy $m$-frame target chunk $\mathbf{Z}_{i, \tau}$ asymmetrically cross-attends to the condition $\mathbf{c}_i$, absorbing temporal physics without corrupting the clean history representations:
\begin{equation}
    \mathbf{Z}_{i, \tau}^{(l+1)} = \mathbf{Z}_{i, \tau}^{(l)} + \text{Attn}\Big(\text{Q}=\text{adaLN}(\mathbf{Z}_{i, \tau}^{(l)}, \tau), \; \text{K,V}=[\text{adaLN}(\mathbf{Z}_{i, \tau}^{(l)}, \tau); ~\mathbf{c}_i]\Big).
\end{equation}
The deeper layers then transition to single-stream blocks ($l\in[L_{\text{d}}+1, L_{\text{d}}+L_{\text{s}}])$, operating exclusively on the target chunk via joint intra-chunk self-attention to focus entirely on local spatiotemporal geometry refinement:
\begin{equation}
    \mathbf{Z}_{i, \tau}^{(l+1)} = \mathbf{Z}_{i, \tau}^{(l)} + \text{Self-Attn}\Big(\text{Q,K,V}=\text{adaLN}(\mathbf{Z}_{i, \tau}^{(l)}, \tau)\Big).
\end{equation}

\subsection{Latent Flow Forcing Curriculum}\label{sec:training}
Autoregressive models trained exclusively on ground-truth histories suffer from severe exposure bias during inference. In our high-dimensional geometry space, predicting a sequence of chunks leads to compounding errors: a slight deviation in the predicted chunk $\hat{\mathbf{Z}}_i$ corrupts the subsequent condition $\mathbf{c}_{i+1}$, rapidly pushing the rollout trajectory off the true geometry manifold. To guarantee stable long-horizon rollouts, we propose a two-stage latent flow forcing curriculum.

\noindent\textbf{Stage 1: Teaching-Forcing Training.} In the first stage, we establish the base continuous-time vector field using oracle historical conditions. Specifically, the context is set as $\mathbf{c}_i = \mathbf{Z}_{i-1}$, where $\mathbf{Z}_{i-1}$ denotes the ground-truth VGGT latent extracted from previous frames. As established in Section \ref{sec:flow}, we optimize $F_\theta$ via a weighted $z$-prediction scheme to mitigate the extreme variance of the high-dimensional latent space. The Stage 1 objective is purely teacher-forced:
\begin{equation}
    \mathcal{L}_{\text{S1}} = \mathbb{E}_{\mathbf{Z}_i, \epsilon, \tau} \left[ \left\| \frac{F_\theta(\mathbf{Z}_{i,\tau}, \tau, \mathbf{c}_i) - \mathbf{Z}_{i,\tau}}{1-\tau} - \frac{\mathbf{Z}_i - \mathbf{Z}_{i,\tau}}{1-\tau} \right\|_2^2 \right],
\end{equation}
However, at test time the model must condition on its own predicted history $\mathbf{c}_i =\hat{\mathbf{Z}}_{i-1}$ rather than oracle latents, creating a train-test mismatch. As a result, errors in the context compound across rollout, making the learned operator increasingly brittle over long horizons.

 \noindent\textbf{Stage 2: Trajectory-Consistent Flow Forcing.} To mitigate this gap, in Stage 2, we expose the model to structured errors from its own rollout. Despite prior attempts~\cite{chen2024diffusion,huang2025self} to mitigate autoregressive drift, existing strategies are not ideal for our setting: analytic Gaussian corruption tends to move weakly compressed 1024-D VGGT latents off the geometry manifold, while fully self-generated histories are often slow in training. We therefore use partially denoised states from the model’s own ODE rollout as structured perturbations, yielding a smoother transition from oracle conditioning to self-conditioning.
 
\noindent \textbf{Partially-denoised Rollout State.} To predict the target chunk $\mathbf{Z}_{i+1}$, we first roll out the preceding chunk $\mathbf{Z}_i$ from noise $\mathbf{Z}_{i, 1} \sim \mathcal{N}(\mathbf{0}, \mathbf{I})$ under its clean history $\mathbf{c}_i=\mathbf{Z}_{i-1}$, but stop the ODE integration using our Euler solver at an intermediate flow time $\tau_{\text{mid}} \in (0,1)$:
\begin{equation}
    \hat{\mathbf{Z}}_{i, \tau_{\text{mid}}} = \text{ODESolve}\Big(F_\theta, \mathbf{Z}_{i, 1}, \mathbf{c}_i, \tau \in [1 \to \tau_{\text{mid}}]\Big).
\end{equation} 
This intermediate state $\hat{\mathbf{Z}}_{i, \tau_{\text{mid}}}$ is a partially denoised rollout. 
To formulate the condition for the next chunk $i+1$, we do not fully replace the ground truth immediately. We construct a mixed condition by interpolating this partially integrated trajectory with the clean ground-truth chunk $\mathbf{Z}_i$ using $\lambda\in[0,1]$:
\begin{equation}
    \mathbf{c}_{i+1}^{\text{mix}}(\lambda, \tau_{\text{mid}}) = (1-\lambda)\mathbf{Z}_i + \lambda\hat{\mathbf{Z}}_{i, \tau_{\text{mid}}}.
\end{equation}
Finally, we optimize $F_\theta$ to predict the clean target chunk $\mathbf{Z}_{i+1}$ conditioned on this mixed, partially denoised history. The target chunk itself is perturbed to a flow time $\tau$, yielding the Stage 2 finetuning objective:
\begin{equation}
    \mathcal{L}_{\text{S2}} = \mathbb{E}_{i, \lambda, \tau, \tau_{\text{mid}}} \left[ \left\| \mathbf{Z}_{i+1} - F_\theta(\mathbf{Z}_{i+1,\tau}, \tau, \mathbf{c}_{i+1}^{\text{mix}}(\lambda, \tau_{\text{mid}})) \right\|_2^2 \right].
\end{equation}
 \noindent \textbf{Curriculum Stability.} In Appendix, we further show theoretical analysis of the proposed flow forcing strategy establishes a sequential ELBO on the uncorrupted data likelihood (detailed in Theorem 1). However, the stability of this bound relies heavily on the variance of the rollout error, $\mathbf{E}_i = \hat{\mathbf{Z}}_{i, \tau_{\text{mid}}} - \mathbf{Z}_i$. Because our mixed condition expands to $\mathbf{c}_{i+1}^{\text{mix}} = \mathbf{Z}_i + \lambda \mathbf{E}_i$, the variance injected into the conditioning signal scales proportionally with $\lambda^2 \mathbb{E}[\|\mathbf{E}_i\|^2]$. Early in training, model can produce unstructured rollouts, meaning $\mathbb{E}[\|\mathbf{E}_i\|^2]$ is prohibitively large. Setting $\lambda$ too high would inject massive variance, completely violating the step-wise KL penalty and causing the ELBO to collapse. To enforce convergence and maintain a tight lower bound, we introduce a \textit{curriculum learning} schedule for the mixing parameter $\lambda$. We linearly anneal $\lambda$ from $0$ to $1$ as finetuning progresses. As the intrinsic rollout error decays, smoothly increasing $\lambda$ toward $1$ progressively exposes the network to its inference-time distribution.
\section{Experiments}
\noindent\textbf{Dataset.}
We assess our approach using the Cityscapes \cite{DBLP:conf/cvpr/CordtsORREBFRS16}, KITTI \cite{DBLP:journals/ijrr/GeigerLSU13} and TartanAir \cite{DBLP:conf/iros/WangZWHQWHKS20} datasets, offering video sequences from urban driving environments and diverse 6-DoF camera trajectories in photorealistic simulated scenes. Cityscapes includes 2,975 training sequences and 500 validation sequences, each containing 30 frames captured at 16 FPS with a resolution of $1024 \times 2048$ pixels. 
KITTI contains 143 training sequences and 13 validation sequences, with a native resolution of $512 \times 1382$ and a frame rate of 10 FPS. TartanAir is a large-scale synthetic dataset rendered in Unreal Engine, covering diverse indoor and outdoor environments, with approximately 700k frames in total. Following the split protocol of Gen3R~\cite{DBLP:journals/corr/abs-2601-04090}, we randomly use 90\% of the data for training and the remaining 10\% for validation.

\noindent \textbf{Implementation Details.}
For the two-stage training, we adopt AdamW with a learning rate of $2\times10^{-4}$ and a weight decay of 0.05, and apply gradient clipping with $\ell_2$ norm of 1.0. Training is conducted on 1 RTX Pro 6000 (96\,GB) GPU. The batch size is set to 64 for Teaching-Forcing Training and 48 for Trajectory-Consistent Flow Forcing. We set the initial chunk size to 4 and perform autoregressive rollout with a stride of 1. The flow transformer is alike to\cite{DBLP:journals/corr/abs-2506-15742}, comprising $L_{d}=8$ double-stream Transformer blocks and $L_{s}=8$ single-stream Transformer blocks. We retain the original intermediate feature dimensionality of 1024 without feature compression. Additional details are provided in Appendix.

\noindent \textbf{Evaluation Metrics.} For depth prediction, we report the mean Absolute Relative Error (AbsRel) and depth accuracy ($\delta_1$). Following prior future prediction works~\cite{karypidis2025dinoforesight,DBLP:journals/corr/abs-2507-19468}, we evaluate on Cityscapes under short-term (3 frames, 187.5ms) and mid-term (9 frames, 562.5ms) forecasting settings, and on KITTI under short-term (2 frames, 200 ms) and mid-term (6 frames, 600 ms) settings. 
To assess 3D point cloud quality, following the evaluation protocol of~\cite{DBLP:journals/corr/abs-2601-04090}, we first align the predicted point clouds to the ground truth using Umeyama similarity alignment \cite{DBLP:journals/pami/Umeyama91}, then sample 20k points from both sets via Farthest Point Sampling (FPS)~\cite{DBLP:conf/nips/QiYSG17}, and compute standard geometric metrics including Accuracy, Completeness, and Chamfer Distance (CD) \cite{DBLP:conf/cvpr/FanSG17}.

\begin{table}[t]
\caption{\textbf{Depth forecasting on KITTI.} 
We evaluate representative world-model baselines on future scene depth prediction, reporting AbsRel ($\downarrow$) and $\delta_1$ ($\uparrow$) for short- and mid-term horizons. 
}
\label{tab:Kitti}
\centering
\small
\setlength{\tabcolsep}{6pt}
\renewcommand{\arraystretch}{1.05}
\resizebox{\linewidth}{!}{
\begin{tabular}{l cccc cccc cccc}
\toprule
\multirow{3}{*}{Method}
& \multicolumn{4}{c}{\texttt{2011\_09\_26\_drive\_0002\_sync}} & \multicolumn{4}{c}{\texttt{2011\_10\_03\_drive\_0047\_sync}} & \multicolumn{4}{c}{Mean (Dataset Avg.)} \\
\cmidrule(lr){2-5}\cmidrule(lr){6-9}\cmidrule(lr){10-13}
& \multicolumn{2}{c}{AbsRel ($\downarrow$)} & \multicolumn{2}{c}{$\delta_1$ ($\uparrow$)}
& \multicolumn{2}{c}{AbsRel ($\downarrow$)} & \multicolumn{2}{c}{$\delta_1$ ($\uparrow$)}
& \multicolumn{2}{c}{AbsRel ($\downarrow$)} & \multicolumn{2}{c}{$\delta_1$ ($\uparrow$)} \\
\cmidrule(lr){2-3}\cmidrule(lr){4-5}\cmidrule(lr){6-7}\cmidrule(lr){8-9}\cmidrule(lr){10-11}\cmidrule(lr){12-13}
& Short & Mid & Short & Mid
& Short & Mid & Short & Mid
& Short & Mid & Short & Mid \\
\midrule
Cosmos-4B~\cite{DBLP:journals/corr/abs-2501-03575}          & 0.138 & 0.148  &  84.0 & 81.4 & 0.155 & 0.156 & 77.5 & 75.2 & 0.155 &  0.165 & 77.1 & 74.1  \\
Cosmos-12B~\cite{DBLP:journals/corr/abs-2501-03575}          & 0.131 & 0.139 & 84.5 & 79.7 & 0.149 & 0.176 & 77.6 & 75.9 & 0.154 & 0.185 & 77.5 & 71.5 \\
\midrule
Copy Last {\scriptsize (VGGT~\cite{wang2025vggt})}              & 0.112 & 0.125 & 89.6 & 82.6 & 0.140 & 0.153 & 84.4 & 79.6 & 0.137 & 0.144 & 83.7 &  78.5\\
Aether~\cite{DBLP:journals/corr/abs-2503-18945}                 & \cellcolor{third!80}0.074 & \cellcolor{third!80}0.102 & \cellcolor{third!80}94.2 & \cellcolor{third!80}88.0 & \cellcolor{third!80}0.113  & \cellcolor{third!80}0.111 & \cellcolor{third!80}87.4 & \cellcolor{third!80}85.2 & \cellcolor{third!80}0.096 & \cellcolor{third!80}0.124 & \cellcolor{third!80}87.9 & \cellcolor{third!80}84.7 \\
DINO-Foresight~\cite{karypidis2025dinoforesight}        & \cellcolor{second!80}0.063 & \cellcolor{second!80}0.105 & \cellcolor{second!80}96.5 & \cellcolor{second!80}89.9 & \cellcolor{second!80}0.076 & \cellcolor{second!80}0.107 & \cellcolor{second!80}90.9 & \cellcolor{second!80}86.7 & \cellcolor{second!80}0.082 & \cellcolor{second!80}0.115 & \cellcolor{second!80}91.6 & \cellcolor{second!80}85.1 \\
\midrule
VGGT-World                & \cellcolor{best!80}\textbf{0.050} & \cellcolor{best!80}\textbf{0.078} & \cellcolor{best!80}\textbf{98.3}  &  \cellcolor{best!80}\textbf{93.2} &  \cellcolor{best!80}\textbf{0.064} & \cellcolor{best!80}\textbf{0.101}  & \cellcolor{best!80}\textbf{94.9}  & \cellcolor{best!80}\textbf{92.7} & \cellcolor{best!80}\textbf{0.065} & \cellcolor{best!80}\textbf{0.098} & \cellcolor{best!80}\textbf{94.0} &\cellcolor{best!80}\textbf{89.4} \\
\bottomrule
\end{tabular}
}
\end{table}

\begin{figure}[t]
    \centering
    \includegraphics[width=\linewidth,height=0.6\textheight,keepaspectratio]{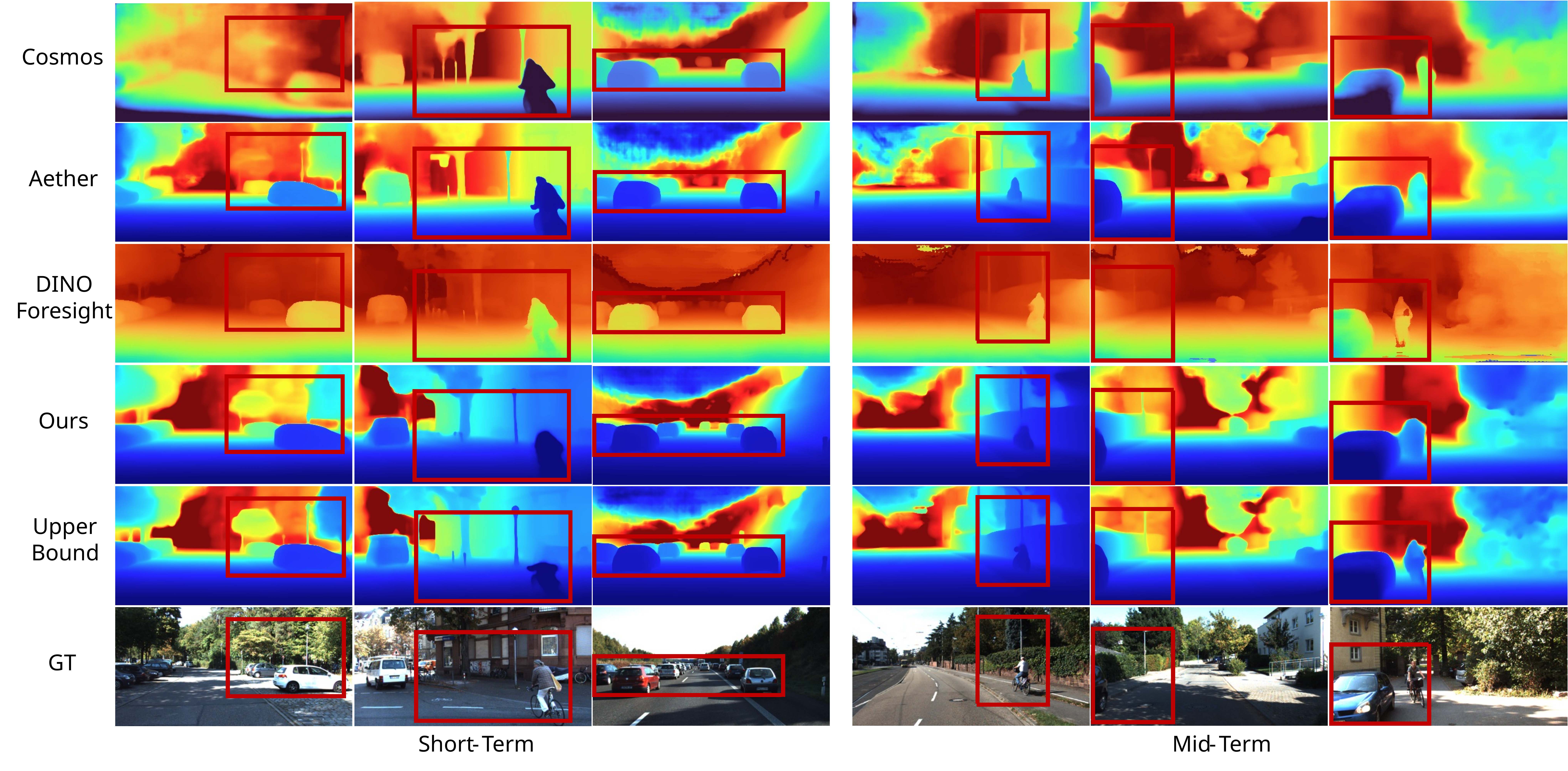}\vspace{-2ex}
    \caption{\textbf{The short- and mid-term depth forecasting results on KITTI.} ``Upper Bound'' denotes the depth obtained by feeding the real future images into VGGT. ``DINO-Foresight'' predicts depth in the scale of Depth Anything~\cite{DBLP:conf/nips/YangKH0XFZ24}, leading to a different visualization appearance compared to others.}\vspace{-4ex}
    \label{fig:kitti_vis}
\end{figure}

\noindent\textbf{Baselines.}
We compare our method with three categories of state-of-the-art world models: (1) \textit{pixel-space 2D generation models}, including Cosmos 4B/12B \cite{DBLP:journals/corr/abs-2501-03575} and VISTA~\cite{DBLP:conf/nips/GaoY0CQ0Z024} (we use its fine-tuned results following~\cite{karypidis2025dinoforesight}); (2) \textit{latent-space 2D generation models}, represented by DINO-Foresight~\cite{karypidis2025dinoforesight}; and (3) \textit{geometry-aware generation models}, such as Aether~\cite{DBLP:journals/corr/abs-2503-18945}, WVD~\cite{DBLP:conf/cvpr/ZhangZMMTSG25}, and Gen3R~\cite{DBLP:journals/corr/abs-2601-04090}. Note that we also include a \textsc{Copy-Last} baseline, where the depth predicted from the ground-truth RGB of the previous frame.

\subsection{Depth Forecasting}

\noindent\textbf{Results on KITTI.}
Table~\ref{tab:Kitti} and Fig.~\ref{fig:kitti_vis} compare VGGT-World with representative world models on depth forecasting. \textit{Quantitatively}, our method consistently achieves the best performance across all horizons, attaining a mean $\delta_1$ of 94.0\%/89.4\% and reducing AbsRel to 0.065/0.098, corresponding to relative improvements of 21\%/15\% on the short- and mid-term horizons respectively. Notably, the pixel-space Cosmos (4B/12B) \cite{DBLP:journals/corr/abs-2501-03575} underperforms even the naive \textsc{Copy-Last} baseline \cite{wang2025vggt}, indicating that scaling video generation without 3D priors prioritizes photometric realism over geometric consistency. While Aether~\cite{DBLP:journals/corr/abs-2503-18945} incorporates geometry supervision, it still underperforms due to pixel-space representation. Additionally, although the latent-space DINO-Foresight~\cite{karypidis2025dinoforesight} achieves competitive accuracy, its performance is largely driven by fine-tuning on ground-truth depth. In contrast, our method, trained in the geometry latent space, achieves superior performance without any fine-tuning on ground-truth depth. \textit{Qualitatively}, our model produces more faithful depth maps shown in Fig.~\ref{fig:kitti_vis}, mitigating distant-region errors and geometric distortions observed in Cosmos \cite{DBLP:journals/corr/abs-2501-03575} and Aether \cite{DBLP:journals/corr/abs-2503-18945}, while preserving sharp structural details and temporal-geometric consistency, demonstrating stronger geometry modeling capability.

\noindent\textbf{Results on Cityscapes.} On Cityscapes, the main challenge is not only short-\begin{wraptable}{r}{0.5\linewidth}\vspace{-4ex}
\caption{\textbf{Depth forecasting on Citys-} \textbf{capes} with AbsRel and depth accuracy $\delta_1$.}\vspace{2ex}
\label{tab:cityscapes}
\centering
\small
\resizebox{\linewidth}{!}{
\renewcommand{\arraystretch}{1.05}
\begin{tabular}{l cc cc}
\toprule
\multirow{2}{*}{Method}
& \multicolumn{2}{c}{AbsRel ($\downarrow$)}
& \multicolumn{2}{c}{$\delta_1$ ($\uparrow$)} \\
\cmidrule(lr){2-3}\cmidrule(lr){4-5}
& Short & Mid & Short & Mid \\
\midrule
Cosmos-4B~\cite{DBLP:journals/corr/abs-2501-03575}  & 0.189 & 0.247 & 76.3 & 70.5 \\
Cosmos-12B~\cite{DBLP:journals/corr/abs-2501-03575} & 0.181 & 0.241 & 78.2 & 72.4 \\
\midrule
\textsc{Copy-Last} {\scriptsize (VGGT~\cite{wang2025vggt})}  & 0.154 & 0.212 & 84.1 & 77.8 \\
VISTA$_{\mathrm{ft}}$~\cite{DBLP:conf/nips/GaoY0CQ0Z024}& 0.124 & 0.153 & 86.4 & 82.8 \\
DINO-Foresight~\cite{karypidis2025dinoforesight} 
& \cellcolor{third!80}0.114 
& \cellcolor{third!80}0.136 
& \cellcolor{third!80}88.6 
& \cellcolor{third!80}85.4 \\
\midrule
VGGT-World {\scriptsize w/o Forcing}
& \cellcolor{second!80}0.079 
& \cellcolor{second!80}0.131 
& \cellcolor{second!80}93.7 
& \cellcolor{second!80}86.8 \\

VGGT-World
& \cellcolor{best!50}\textbf{0.078} 
& \cellcolor{best!50}\textbf{0.126} 
& \cellcolor{best!50}\textbf{93.8} 
& \cellcolor{best!50}\textbf{87.3} \\
\midrule
VGGT-World ($\lambda=0.1$) & 0.084 & 0.142 & 93.1 & 85.0 \\
VGGT-World ($\lambda=0.3$) & 0.089 & 0.148 & 92.6 & 84.9 \\
VGGT-World ($\lambda=0.7$) & 0.109 & 0.165 & 90.1 & 81.4 \\
\bottomrule
\end{tabular}\vspace{-3ex}
}
\end{wraptable}term frame extrapolation, but maintaining metrically stable depth under continuous ego-motion and large perspective changes.  This is reflected by the larger gap between short- and mid-term results across all methods. As shown in Table~\ref{tab:cityscapes}, our method achieves the best performance on both metrics across short- and mid-term horizons, improving AbsRel and $\delta_1$ by up to 32\%. Mid-term forecasting is inherently more challenging, as small structural errors in earlier predictions can accumulate and lead to larger depth deviations. More qualitative results are provided in the Appendix for reference.

\subsection{Point Map Forecasting}
Table~\ref{tab:Tartanair} provides point cloud forecasting results on TartanAir, demonstrating that \textsc{VGGT-World} achieves superior geometry accuracy by predicting scene directly within geometry-centric latent spaces. Compared to pixel-space generative model \textsc{Gen3R} which models dynamics in pixel space, our method significantly outperforms it on all metrics in the pose/text-free setting, while remaining competitive even when it is aided by explicit pose conditions. Furthermore, the performance gain of the full model over the w/o Forcing variant validates that our Trajectory-Consistent Flow Forcing effectively mitigates exposure bias during rollout. These results confirm that by repurposing spatiotemporal features rather than reconstructing pixels, \textsc{VGGT-World} ensures superior geometric fidelity at a fraction of the computational cost of existing generative baselines. Qualitative results in Fig. \ref{fig:tartanair} further illustrate the superior geometric fidelity of \textsc{VGGT-World}. Without pose/text conditioning, \textsc{Gen3R} exhibits significant structural disorganization and blurry distortions, failing to maintain scene integrity. In contrast, our method produces well-defined geometry that closely adheres to the ground-truth topology across both indoor and outdoor scenes. 
\begin{table}[t]
\centering
\caption{\textbf{Geometry generation on TartanAir} with different conditioning (text, pose) under 1-view and 2-view settings.}
\label{tab:Tartanair}\vspace{-1ex}
\small
\setlength{\tabcolsep}{5pt}
\renewcommand{\arraystretch}{1.12}
\resizebox{0.95\linewidth}{!}{%
\begin{tabular}{l cc ccc ccc}
\toprule
\multirow{2}{*}{Method}
& \multicolumn{2}{c}{Condition}
& \multicolumn{3}{c}{1-view}
& \multicolumn{3}{c}{2-views} \\
\cmidrule(lr){2-3} \cmidrule(lr){4-6} \cmidrule(lr){7-9}
& Text-free & Pose-free
& Acc.$\downarrow$ & Comp.$\downarrow$ & CD$\downarrow$
& Acc.$\downarrow$ & Comp.$\downarrow$ & CD$\downarrow$ \\
\midrule
\textsc{Aether} \cite{DBLP:journals/corr/abs-2503-18945} & \cmark& \xmark
& 3.1547 & 4.5366 & 3.8457
& 2.7745 & 3.4420 & 3.1082 \\
\textsc{WVD} \cite{DBLP:conf/cvpr/ZhangZMMTSG25} & \cmark & \xmark
& 4.3944 & 3.0660 & 3.7302
& 4.3794 & 2.5268 & 3.4531 \\
\textsc{Gen3R} \cite{DBLP:journals/corr/abs-2601-04090} & \xmark & \xmark
& 3.0250 & 2.5367 & 2.7809
& 2.2825 & 1.6643 & 1.9734 \\
\textsc{Gen3R} \cite{DBLP:journals/corr/abs-2601-04090} & \cmark& \xmark
& 3.4579 & 4.7384 & 4.0982
& \cellcolor{third!80}{2.6029} & 3.1154 & \cellcolor{second!80}2.8591 \\
\textsc{Gen3R} \cite{DBLP:journals/corr/abs-2601-04090} & \cmark & \cmark
& 3.6427 & 7.8528 & 5.7478
& 2.7441 & 6.1423 & 4.4432 \\
\midrule
VGGT-World {\scriptsize w/o Forcing} & \cmark & \cmark
& 1.2200 & 5.0435 & 3.0818
& \cellcolor{second!80}{1.2200} & 5.0435 & \cellcolor{third!80}{3.0818} \\
VGGT-World & \cmark & \cmark
& 1.1604 & 4.5297 & 2.8451
& \cellcolor{best!50}\textbf{1.1604} & 4.5297 & \cellcolor{best!50}\textbf{2.8451} \\
\bottomrule
\end{tabular}}\vspace{-1ex}
\end{table}

\begin{figure}[t]
    \centering
    \includegraphics[width=0.98\linewidth]{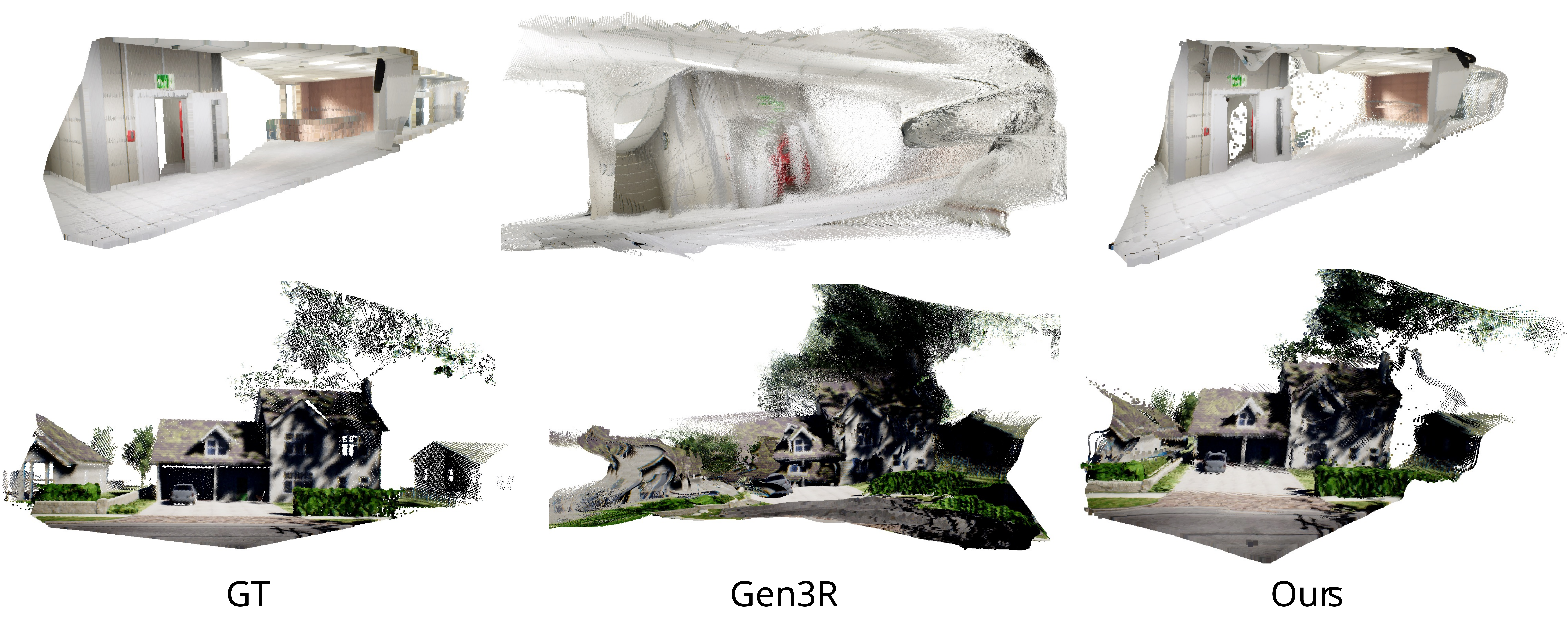}\vspace{-2ex}
    \caption{\textbf{Point cloud forecasting on TartanAir}. Gen3R produces structurally disorganized geometry on walls and rooftops, whereas our method preserves coherent structure and yields more accurate predictions.}\vspace{-1ex}
    \label{fig:tartanair}
\end{figure}

\subsection{Further Analysis}
\noindent\textbf{Impact of Forcing Curriculum.} 
Bottom part of Table~\ref{tab:cityscapes} evaluates the impact of the mixing ratio $\lambda$, where our \textsc{VGGT-World} adopts a linear scheduling strategy that gradually increases $\lambda$ as training progresses. Although the model without forcing achieves strong short-term performance, it degrades during rollout. The proposed dynamic scheduling alleviates this issue by progressively enhancing robustness to recursive conditioning while preserving the learning of scene dynamics. In contrast, static settings cause the model to learn a fixed and biased perturbation distribution; moreover, a static large mixing rate (\eg, $\lambda=0.7$) can induce catastrophic forgetting of the Stage-1 distribution. Overall, the curriculum-based design balances stable structural learning with long-horizon forecasting stability.


\begin{figure}[t]
    \centering
    \begin{minipage}[t]{0.62\linewidth}
        \centering
        \vspace{3ex}
        \includegraphics[width=\linewidth]{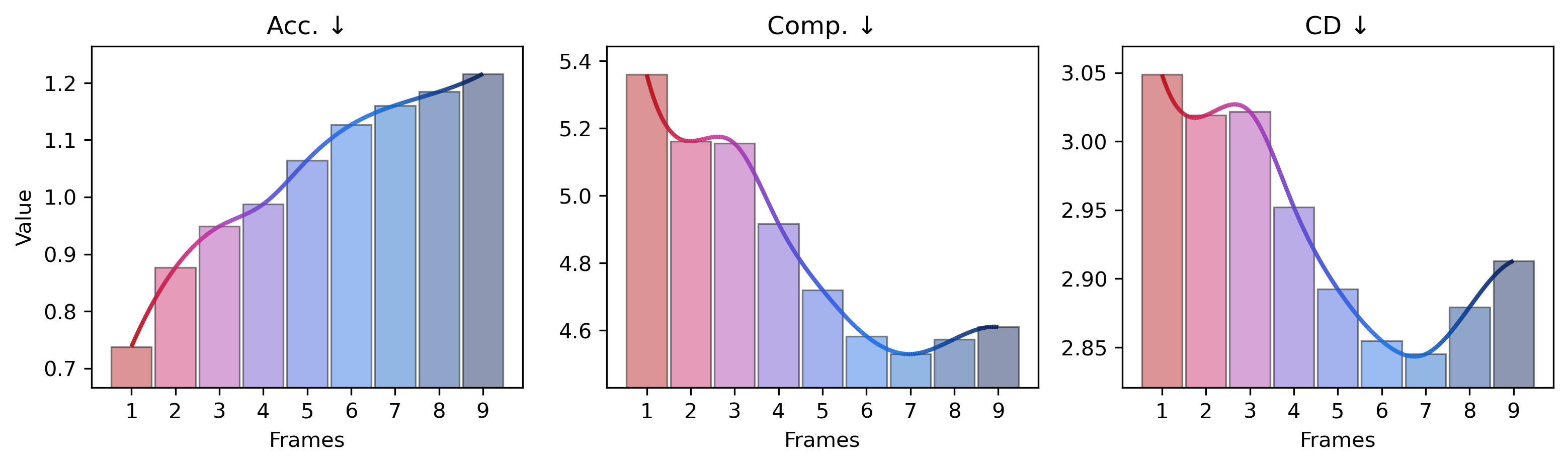}\vspace{-2ex}
        \caption{Long-term horizon forecasting on TartanAir.}
        \label{fig:tartanair_long}
    \end{minipage}
    \hfill
    \begin{minipage}[t]{0.34\linewidth}
        \centering
        \small
        \setlength{\tabcolsep}{6pt}
        \renewcommand{\arraystretch}{1.1}
        \captionof{table}{Short-term horizon foresight on TartanAir.}\vspace{1ex}
        \label{tab:shortterm_compare}
        \resizebox{1\linewidth}{!}{%
        \begin{tabular}{lcc}
        \toprule
        Metric & VGGT \cite{wang2025vggt} & VGGT-World \\
        \midrule
        Acc. $\downarrow$  & \textbf{0.8672} & 0.8764 \\
        Comp. $\downarrow$ & 1.2159 & \textbf{1.0646} \\
        CD $\downarrow$    & 1.0416 & \textbf{0.9705} \\
        \bottomrule
        \end{tabular}}
    \end{minipage}\vspace{-3ex}
\end{figure}

\noindent\textbf{Can future prediction improve present-time perception?}
We further ask whether \textsc{VGGT-World} can provide a better geometric estimation for the \emph{current} moment than directly applying VGGT to the current input alone. The short-horizon comparison in Table~\ref{tab:shortterm_compare} shows that this is largely the case. Although the accuracy metric changes only marginally, \textsc{VGGT-World} achieves clearly better completeness and a lower overall CD value, indicating a more balanced and globally consistent reconstruction. This result suggests that predictive world modeling does not merely support future imagination but can also feed back to strengthen present-time perception. For embodied agents, such a capability is particularly appealing: instead of treating perception as a purely instantaneous process, the agent can leverage short-term foresight to refine its current geometric belief, leading to more reliable downstream planning and control.

\noindent\textbf{Camera Parameter Forecasting.} Beyond depth and point cloud forecasting, \begin{wrapfigure}{r}{0.5\linewidth}
\vspace{-2ex}
\centering
\includegraphics[width=\linewidth]{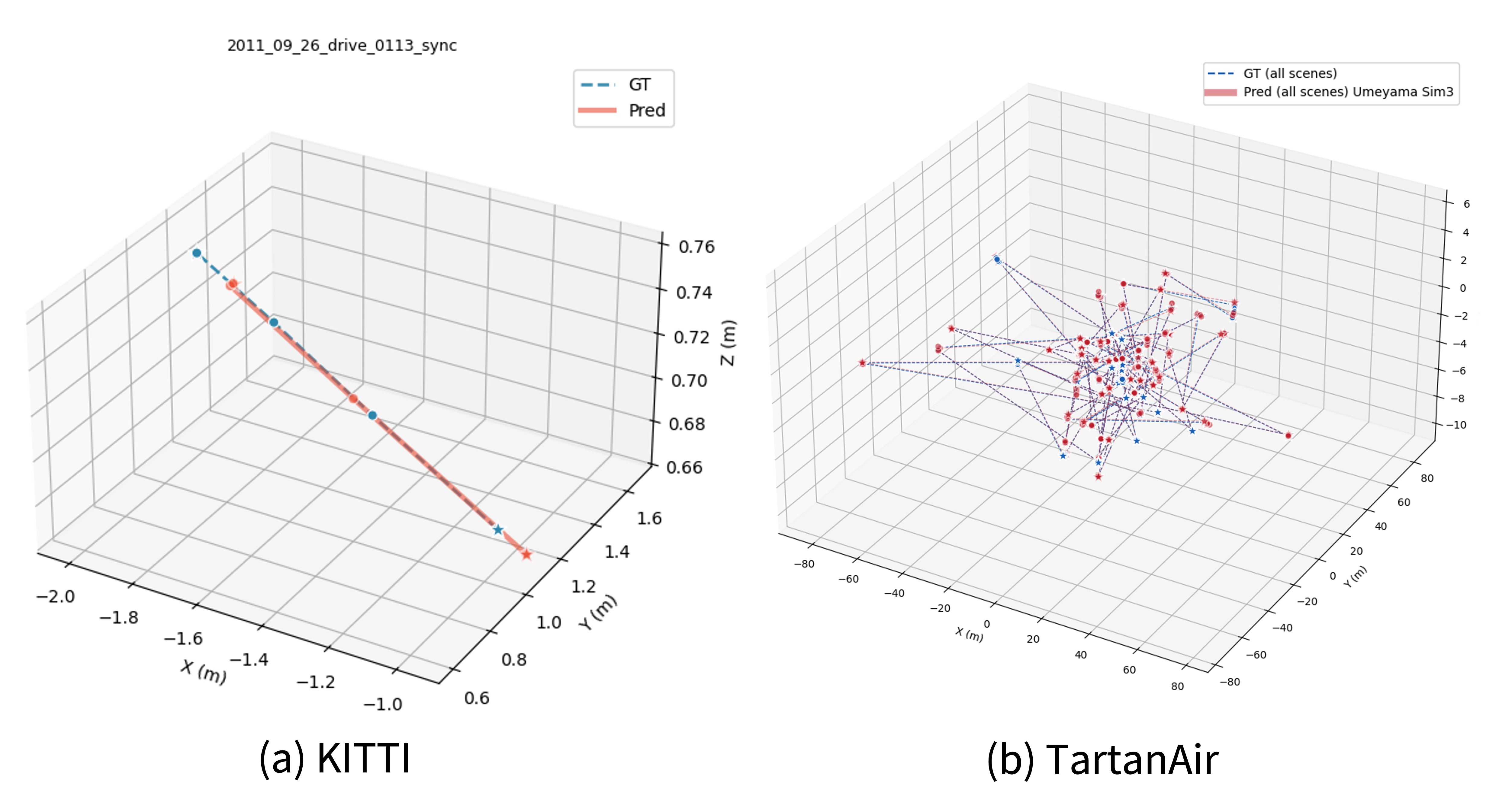}\vspace{-2ex}
\caption{Qualitative study of camera parameter prediction on KITTI and TartanAir.}
\label{fig:traj_vis}
\captionsetup{type=table}
\label{tab:traj_eval}
\vspace{-3ex}
\end{wrapfigure}we further evaluate whether the predicted next-frame VGGT features preserve camera motion cues. We predict camera parameters from the generated VGGT features and compare them against ground truth trajectories on KITTI and TartanAir. After Umeyama alignment, the predicted trajectories exhibit relatively consistent global structure with the ground truth in Fig.~\ref{fig:traj_vis}. Quantitative analysis of ATE/RTE/RRE errors is reported in the Appendix.
These results suggest that VGGT-World's latent dynamics preserve both scene structure and camera evolution without explicit pose supervision, achieving reasonable consistency without the prohibitive computational costs of appearance-driven world models.

\noindent\textbf{How far into the future should we forecast?}
We first ask whether forecasting a longer future horizon can further benefit geometry estimation when the predicted observations are aggregated with the current frame and passed to VGGT. As shown in Fig.~\ref{fig:tartanair_long}, extending the foresight horizon initially improves the reconstructed geometry, suggesting that moderate future context provides complementary structural cues beyond the instantaneous observation. However, this benefit is not monotonic. The CD reaches its best value around frame 7, indicating a practical sweet spot for horizon length. Beyond this point, performance begins to degrade, and the overall geometry exhibits distortion.

\noindent\textbf{{Efficiency Analysis.}} We evaluate the computational efficiency of \textsc{VGGT-World} against representative video-generation paradigms, including \textsc{Cosmos-12B}~\cite{DBLP:journals/corr/abs-2501-03575} and \textsc{Gen3R}~\cite{DBLP:journals/corr/abs-2601-04090}. Since they operate under varying input and rollout horizons, we adopt a normalized latency metric (seconds per generated frame) to\begin{wrapfigure}{r}{0.6\columnwidth}
    \vspace{-2ex}
    \centering
    \includegraphics[width=1\linewidth]{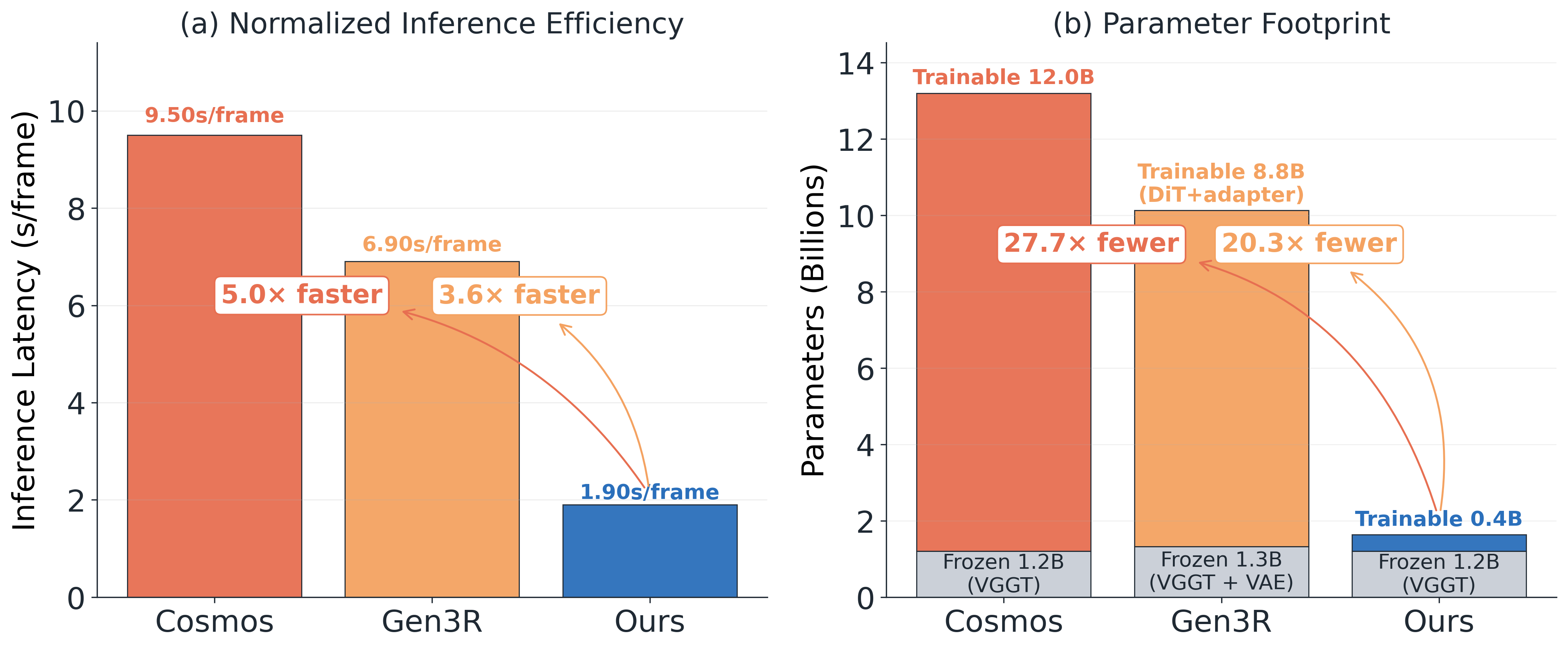}
    \caption{\textbf{Efficiency comparison with VideoGen world models.} (a) Inference latency per generated frame. (b) Trainable vs.\ frozen parameter footprint.}
    \vspace{-4ex}
    \label{fig:efficiency}
\end{wrapfigure} ensure a fair, scale-invariant comparison. 
As shown in Fig.~\ref{fig:efficiency}\textcolor{red}{a}, \textsc{Cosmos-12B} requires 9.50 s/frame to synthesize future frames on KITTI. In contrast, our approach achieves an end-to-end inference time of 1.9 s/frame. This represents a 5$\times$ reduction in per-frame latency compared to \textsc{Cosmos} and a 3.6$\times$ speedup over \textsc{Gen3R} (6.90s/frame).

\noindent\textbf{Trainable Footprint.} 
Beyond inference speed, the practicality of world models in resource-constrained settings depends heavily on the number of trainable parameters. In Fig.~\ref{fig:efficiency}\textcolor{red}{b}, our approach optimizes only 0.43B parameters by keeping the 1.3B VGGT backbone frozen. In contrast, \textsc{Cosmos-12B} requires updating all 12.0B parameters, and \textsc{Gen3R} entails a heavy trainable DiT (2.7B) alongside a geometry adapter. By modeling dynamics directly in the geometry-centric latent space, \textsc{VGGT-World} achieves a much smaller trainable footprint while keeping competitive capacity, enabling easier reproduction and faster iteration.

\section{Conclusion and Future Work}
We presented VGGT-World, an early attempt at world modeling directly in a high-dimensional geometric latent space. Despite its simplicity, our approach already surpasses substantially heavier video-generation world models on both forecasting quality and efficiency, suggesting that geometry-centric latent forecasting is a promising alternative to full-frame generation. Moreover, because VGGT-World requires no ground-truth depth or point-cloud supervision, it provides a flexible foundation for scaling geometry world models. Looking ahead, we can enable richer conditioning by incorporating camera pose and action information into the generative process; for embodied agents, action-conditioned forecasting may further support more controllable and interactive rollouts.

\clearpage  


%
%
\bibliographystyle{splncs04}
\bibliography{main}

\clearpage
\appendix

\appendix

\vspace*{0.5em}
\begin{center}
{\Large\bfseries Supplementary Material}
\end{center}
\vspace{1em}

\paragraph{Summary of the Supplementary Material.}
This supplementary document provides additional theoretical analysis, implementation details, and extended experimental results that complement the main paper.
Section~\ref{sec:theory} presents the formal derivation of the Sequential ELBO under the proposed trajectory-consistent flow forcing objective.
Section~\ref{sec:implementation} details the model architecture, training protocol, and dataset configurations.
Section~\ref{sec:results} reports additional qualitative and quantitative results on Cityscapes, KITTI, and TartanAir.
Finally, Section~\ref{sec:limitations} discusses limitations and directions for future work.

\section{Theoretical Analysis}\label{sec:theory}
In this section, we provide the formal proof for Theorem~1 presented in Section~3.3 of the main text. We demonstrate that optimizing our Trajectory-Consistent Flow Forcing objective rigorously maximizes a Sequential Evidence Lower Bound (ELBO) on the true data likelihood. Furthermore, we provide a theoretical analysis of how the variance of the rollout error interacts with the step-wise Kullback-Leibler (KL) divergence, mathematically justifying the necessity of our $\lambda$ curriculum schedule.

\begin{theorem}[Sequential ELBO under Trajectory-Consistent Flow Forcing]
Let $\mathbf{Z}_{1:S}$ be a sequence of uncorrupted geometry chunks drawn from the true data distribution. Let $\mathbf{C} = (\mathbf{c}_1^{\text{mix}}, \dots, \mathbf{c}_S^{\text{mix}})$ be the sequence of auxiliary interpolated conditions. Assuming strict causality in both the forward corruption process and the generative model, the uncorrupted marginal log-likelihood admits the following evidence lower bound:
\begin{equation}
    \begin{split}
    \log p_\theta(\mathbf{Z}_{1:S}) \ge {} & \sum_{i=1}^S \mathbb{E}_{q_{\lambda}} \Big[ \log p_\theta(\mathbf{Z}_i \mid \mathbf{c}_i^{\text{mix}}) \Big] \\
    & - \sum_{i=1}^S \mathbb{E}_{q_{\lambda}} \Big[ D_{\text{KL}}\big( q_{\lambda}(\mathbf{c}_i^{\text{mix}} \mid \mathbf{Z}_{i-1}, \mathbf{c}_{i-1}^{\text{mix}}) \parallel p_\theta(\mathbf{c}_i^{\text{mix}} \mid \mathbf{c}_{<i}^{\text{mix}}) \big) \Big].
\end{split}
\end{equation}
\end{theorem}
\begin{proof}
  We aim to lower-bound the marginal log-likelihood of the true data sequence $\mathbf{Z}_{1:S}$. Since the explicit joint distribution of the high-dimensional autoregressive sequence is intractable, we treat the sequence of mixed conditions $\mathbf{C}$ as auxiliary latent variables. The marginal log-likelihood can be expressed by integrating out $\mathbf{C}$:
  \begin{equation}
      \log p_\theta(\mathbf{Z}_{1:S}) = \log \int p_\theta(\mathbf{Z}_{1:S}, \mathbf{C}) \, d\mathbf{C}.
  \end{equation}
  We introduce the forward corruption process $q_{\lambda}(\mathbf{C} \mid \mathbf{Z}_{1:S})$, parameterized by the curriculum schedule $\lambda$, which governs how the mixed conditions are generated via our ODE trajectory interpolation. Multiplying and dividing by this forward distribution yields:
  \begin{equation}
    \begin{aligned}
    \log p_\theta(\mathbf{Z}_{1:S}) &= \log \int q_{\lambda}(\mathbf{C} \mid \mathbf{Z}_{1:S}) \frac{p_\theta(\mathbf{Z}_{1:S}, \mathbf{C})}{q_{\lambda}(\mathbf{C} \mid \mathbf{Z}_{1:S})} \, d\mathbf{C} \\
    &= \log \mathbb{E}_{q_{\lambda}} \left[ \frac{p_\theta(\mathbf{Z}_{1:S}, \mathbf{C})}{q_{\lambda}(\mathbf{C} \mid \mathbf{Z}_{1:S})} \right].
    \end{aligned}
  \end{equation}
  Since the logarithmic function is strictly concave, we apply Jensen's Inequality to bring the logarithm inside the expectation, establishing the standard variational lower bound:

\begin{equation}
\begin{aligned}
\log p_\theta(\mathbf{Z}_{1:S})
\ge {} & \mathbb{E}_{q_{\lambda}} \left[ \log \frac{p_\theta(\mathbf{Z}_{1:S}, \mathbf{C})}{q_{\lambda}(\mathbf{C} \mid \mathbf{Z}_{1:S})} \right] \\
&= {} \mathbb{E}_{q_{\lambda}} \Big[ \log p_\theta(\mathbf{Z}_{1:S}, \mathbf{C}) - \log q_{\lambda}(\mathbf{C} \mid \mathbf{Z}_{1:S}) \Big].
\end{aligned}
\end{equation}
To map this bound to our autoregressive generation framework, we must factorize both the generative model and the forward corruption process according to their causal structure. For the generative model $p_\theta$, the joint probability of the data and conditions factorizes step-wise:
\begin{equation}
    p_\theta(\mathbf{Z}_{1:S}, \mathbf{C}) = \prod_{i=1}^S p_\theta(\mathbf{Z}_i \mid \mathbf{c}_i^{\text{mix}}) \, p_\theta(\mathbf{c}_i^{\text{mix}} \mid \mathbf{c}_{<i}^{\text{mix}}).
\end{equation}
For the forward corruption process $q_{\lambda}$, generating $\mathbf{c}_i^{\text{mix}}$ requires partially integrating the ODE solver from pure noise conditioned on the previous state $\mathbf{c}_{i-1}^{\text{mix}}$, and then interpolating the result with the ground-truth chunk $\mathbf{Z}_{i-1}$. Therefore, $q_{\lambda}$ satisfies a first-order Markov property with respect to the immediate past variables:\[
q_{\lambda}(\mathbf{C} \mid \mathbf{Z}_{1:S})
=
\prod_{i=1}^S q_{\lambda}(\mathbf{c}_i^{\text{mix}} \mid \mathbf{Z}_{i-1}, \mathbf{c}_{i-1}^{\text{mix}})
\](Note: For the initial step $i=1$, we define $\mathbf{Z}_0 = \emptyset$ and $\mathbf{c}_0^{\text{mix}} = \emptyset$.) Substituting these autoregressive factorizations back into our lower bound, we obtain:

\begin{equation}
\begin{split}
\log p_\theta(\mathbf{Z}_{1:S})
\ge {} & \mathbb{E}_{q_{\lambda}} \Bigg[
\sum_{i=1}^S \log \Big(
p_\theta(\mathbf{Z}_i \mid \mathbf{c}_i^{\text{mix}})
\, p_\theta(\mathbf{c}_i^{\text{mix}} \mid \mathbf{c}_{<i}^{\text{mix}})
\Big) \\
&\qquad\qquad\quad
- \sum_{i=1}^S \log q_{\lambda}
\big(
\mathbf{c}_i^{\text{mix}} \mid \mathbf{Z}_{i-1}, \mathbf{c}_{i-1}^{\text{mix}}
\big)
\Bigg] \\
= {} & \sum_{i=1}^S \mathbb{E}_{q_{\lambda}} \Big[
\log p_\theta(\mathbf{Z}_i \mid \mathbf{c}_i^{\text{mix}})
\Big] \\
&\qquad
+ \sum_{i=1}^S \mathbb{E}_{q_{\lambda}} \left[
\log \frac{
p_\theta(\mathbf{c}_i^{\text{mix}} \mid \mathbf{c}_{<i}^{\text{mix}})
}{
q_{\lambda}(\mathbf{c}_i^{\text{mix}} \mid \mathbf{Z}_{i-1}, \mathbf{c}_{i-1}^{\text{mix}})
}
\right].
\end{split}
\end{equation}
Recognizing that the second summation represents the negative KL divergence between the artificial forward corruption distribution and the model's learned autoregressive transition prior, we arrive at the final Sequential ELBO:
\begin{equation}
    \begin{split}
        \log p_\theta(\mathbf{Z}_{1:S}) &\ge \sum_{i=1}^S \mathbb{E}_{q_{\lambda}} \Big[ \log p_\theta(\mathbf{Z}_i \mid \mathbf{c}_i^{\text{mix}}) \Big] \\
        & - \sum_{i=1}^S \mathbb{E}_{q_{\lambda}} \Big[ D_{\text{KL}}\big( q_{\lambda}(\mathbf{c}_i^{\text{mix}} \mid \mathbf{Z}_{i-1}, \mathbf{c}_{i-1}^{\text{mix}}) \parallel p_\theta(\mathbf{c}_i^{\text{mix}} \mid \mathbf{c}_{<i}^{\text{mix}}) \big) \Big]. \blacksquare
    \end{split}
\end{equation}
\end{proof}

\section{Implementation Details}\label{sec:implementation}
\subsection{Model Architecture}
Our model is built upon the VGGT~\cite{wang2025vggt} architecture. The geometry latent representation is extracted from the early stages of the VGGT backbone, which consists of a DINO ViT encoder~\cite{DBLP:conf/iccv/CaronTMJMBJ21,DBLP:journals/tmlr/OquabDMVSKFHMEA24} followed by the first five Transformer layers (index $0-4$) of the VGGT aggregator. The remaining Transformer layers (index $5-47$) together with the task-specific heads, including camera head, depth head, and point head, serve as the decoder that maps predicted latent tokens back to geometric outputs. For the condition generation module, we adopt a Flow Transformer architecture similar to~\cite{DBLP:journals/corr/abs-2506-15742}. It consists of 8 double-stream Transformer blocks followed by 8 single-stream Transformer blocks. Throughout the model, the latent dimension is kept at 1024 without applying any feature compression. Instead of predicting noise or velocity fields as in diffusion or flow matching models, our model directly predicts the geometry latent representation $\mathbf{z}$. To encode temporal ordering, we apply 3D Rotary Positional Embedding (3DRoPE) to the concatenated sequence of condition latents and predicted latents, allowing the model to capture the temporal relationships between frames.

\subsection{Evaluation Datasets}
\begin{itemize}
\renewcommand\labelitemi{\textbullet}
\item \textbf{Cityscapes~\cite{DBLP:conf/cvpr/CordtsORREBFRS16}.}
Cityscapes is recorded at 16 FPS with a resolution of $1024 \times 2048$ and a $2{:}1$ aspect ratio. We resize the images by scaling the shorter side and then apply center cropping to obtain a resolution of $224 \times 448$. Each sequence contains 30 frames, indexed by $t \in \{0,1,\ldots,29\}$. Following~\cite{DBLP:journals/corr/abs-2507-19468}, we treat each sequence as a single forecasting sample and use frame 19 as the target frame. 
Taking frames [13, 16] as input, we predict frames 19 and 22 but focus our evaluation on frame 19. Notably, frame 19 occurs 187.5 ms after frame 16, defining the specific short-term horizon for our analysis.
Using frames [7, 10] as input, we predict a sequence up to frame 19 via autoregressive generation. While frames 13 and 16 are generated in the process, performance is measured at frame 19, corresponding to a 562.5 ms look-ahead from the last input frame.

\item \textbf{KITTI~\cite{DBLP:journals/ijrr/GeigerLSU13}.}
KITTI is recorded at 10 FPS with a resolution of $512 \times 1382$, and images are center-cropped to $224 \times 448$. Sequence lengths range from approximately 200 to 1000 frames. Following~\cite{DBLP:journals/corr/abs-2507-19468}, we treat each sequence as a single forecasting sample and use frame 13 as the target frame. 
In the short-term setting, frames [9, 11] are used to predict both frames 13 and 15, but evaluation is restricted solely to frame 13 (200 ms from frame 11). Similarly, for mid-term prediction, we autoregressively generate frames [9, 11, 13] from input [5, 7], while performance is measured only at frame 13 (600 ms from frame 7). We use Cityscapes and KITTI to evaluate short and mid horizon depth forecasting without pose condition.

\item \textbf{TartanAir~\cite{DBLP:conf/iros/WangZWHQWHKS20}.}
TartanAir is a large-scale synthetic dataset containing diverse indoor and outdoor scenes with accurate camera poses and depth annotations. The original RGB images are provided at a resolution of $640 \times 480$ and sampled at 10 FPS. We center-crop the images to $280 \times 280$ for both training and evaluation. Following~\cite{DBLP:journals/corr/abs-2601-04090}, we randomly sample 10\% of the scenes and further sample 49 frames from each scene for evaluation. In practice, we use the aggregated point cloud reconstructed from the first seven frames as the evaluation target for long-horizon point cloud forecasting under pose-free settings. 
We therefore evaluate both short- and long-horizon point cloud forecasting without pose conditions on the TartanAir dataset.

\end{itemize}

\begin{figure}
    \centering
    \includegraphics[width=0.98\linewidth]{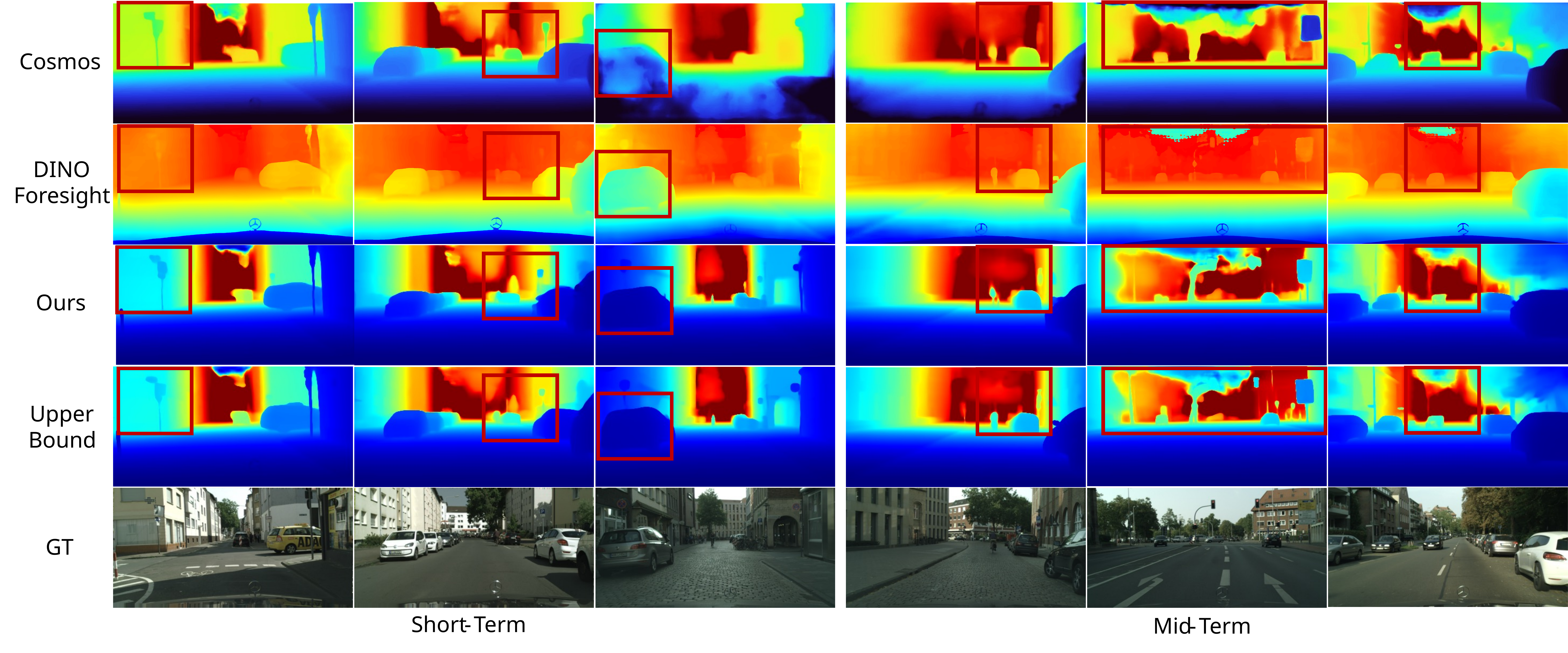}
    \caption{\textbf{The short- and mid-term depth forecasting results on Cityscapes.} ``Upper Bound'' denotes the depth obtained by feeding the real future images into VGGT. ``DINO-Foresight'' predicts depth in the scale of Depth Anything~\cite{DBLP:conf/nips/YangKH0XFZ24}, leading to a different visual appearance from the other methods.}
    \label{fig:app_cityscapes}
\end{figure}

\subsection{Chunk Size Details}
The proposed training strategy consists of two sequential stages. Initially, in the Teacher-Forcing Training stage, we train the model on sampled chunks of size 4, where the first two frames serve as the conditioning context to predict the subsequent two frames as targets. To enhance temporal stability, we introduce a second-stage Trajectory-Consistent Flow Forcing using chunks of length 5 (e.g., [0,1,2,3,4]). In this stage, we first utilize the ground-truth (GT) latents of frames 0 and 1 to predict the latents for frames 2 and 3. Subsequently, we perform a rollout step by combining the GT latent of frame 1 with the previously predicted latent of frame 2 to predict the latents of frames 3 and 4. Crucially, to mitigate exposure bias, gradients are backpropagated exclusively through this subsequent rollout stage. During inference, we maintain a consistent chunk-size-4 configuration and autoregressively generate future frames using a sliding window with a stride of 1.

\begin{figure}[t]
    \centering
    \includegraphics[width=0.98\linewidth]{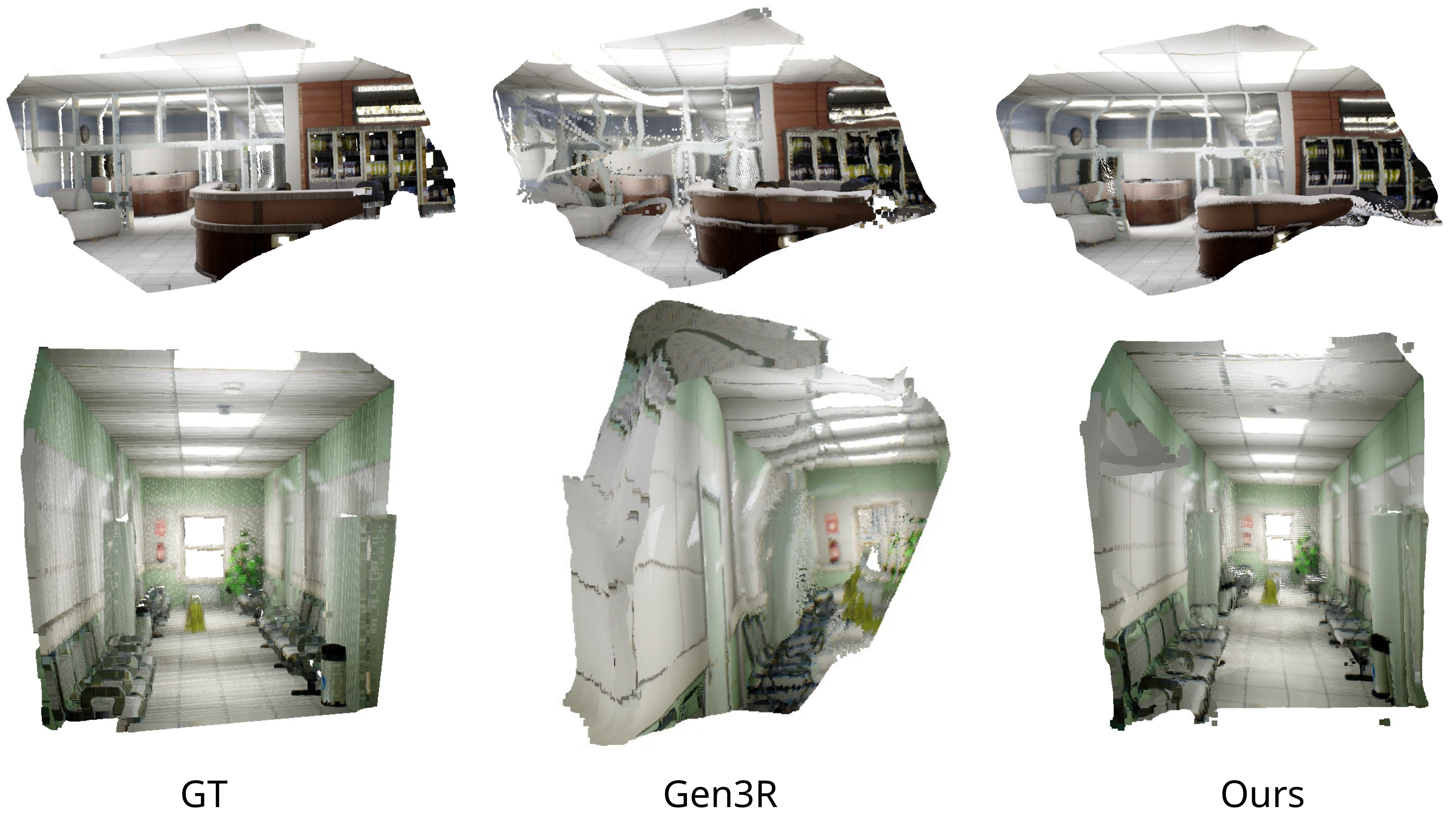}
    \caption{\textbf{Point cloud forecasting on TartanAir.} Gen3R exhibits noticeable artifacts in structurally homogeneous regions such as walls and windows.}
    \label{fig:app_tartanair2}
\end{figure}

\begin{figure}
    \centering
    \includegraphics[width=0.98\linewidth]{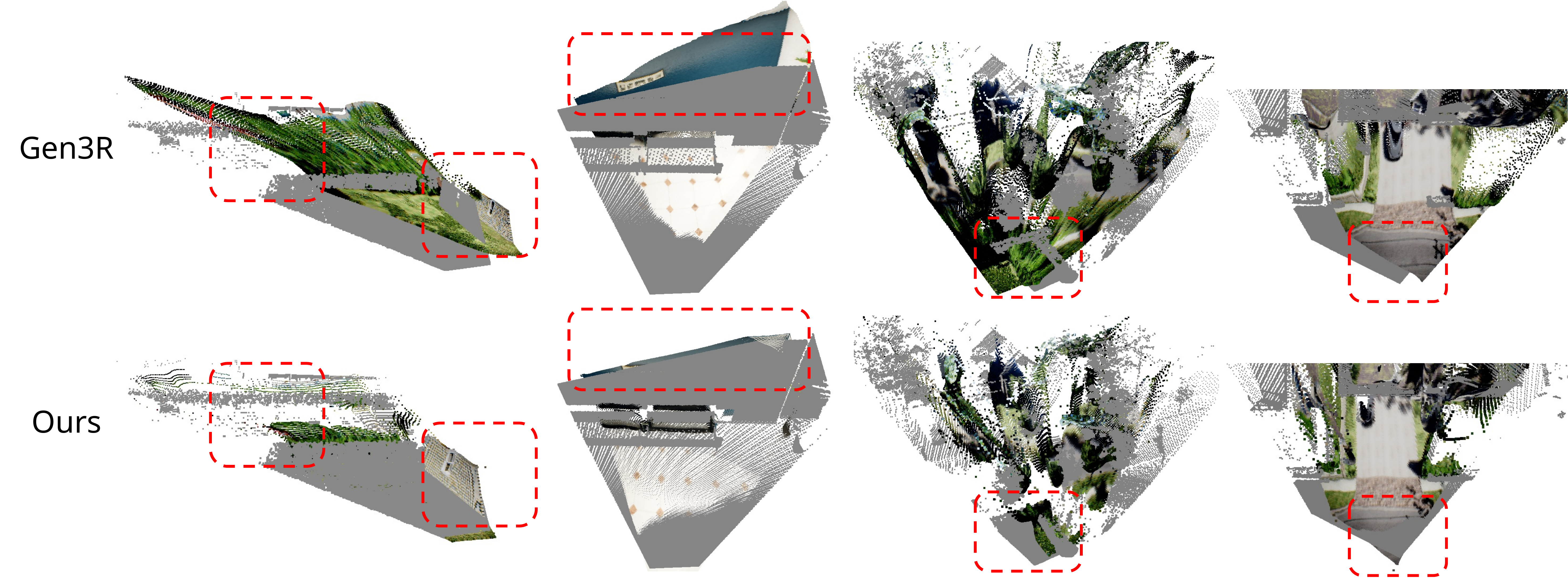}
    \caption{\textbf{Single-frame point cloud forecasting results on TartanAir.} The gray regions in each point cloud indicate the ground truth.}
    \label{fig:app_tartanair1}
\end{figure}

\section{Additional Results}\label{sec:results}
\subsection{Qualitative Results on Cityscapes}
As mentioned in Section~4.1 of the main text, we provide extended qualitative comparisons on Cityscapes in Fig.~\ref{fig:app_cityscapes}. 
Specifically, we compare our method against Cosmos~\cite{DBLP:journals/corr/abs-2501-03575}, DINO-Foresight~\cite{karypidis2025dinoforesight}, and an upper bound derived from VGGT. Cosmos, a video-generation-based model, suffers from severe artifacts in its depth prediction. This indicates that pixel-driven generation tends to prioritize perceptual plausibility over geometric consistency. 
While DINO-Foresight partially alleviates these distortions by operating within a latent space, it still exhibits geometric inaccuracies, particularly in distant regions.
In contrast, our method explicitly models dynamics in the geometric space, leading to the best overall generation quality.

\subsection{Qualitative Results on TartanAir}
As mentioned in Section~4.2 of the main text, we provide additional qualitative comparisons on TartanAir in Fig.~\ref{fig:app_tartanair2} and Fig.~\ref{fig:app_tartanair1}. Fig.~\ref{fig:app_tartanair2} provides additional full-sequence point cloud forecasting results as a supplement to the main text. Based on empirical observation, we use the point cloud forecasting results over the first seven frames for comparison. Without pose conditioning, Gen3R predicts substantial deviations from the original point cloud, which directly lead to failures in global alignment and consequently produce incomplete 3D structures, especially in regions that heavily rely on accurate alignment, such as windows and walls. In contrast, our autoregressive framework follows the trajectory implied by the conditioning frames, enabling point cloud growth along the correct geometric direction and thereby maintaining a more reasonable 3D structure. Fig.~\ref{fig:app_tartanair1} shows a single-frame comparison of the predicted point clouds at the third frame between our method and Gen3R. The gray regions denote the ground-truth geometry of the target frame, while the RGB-colored regions correspond to the predicted point cloud. Because Gen3R performs global alignment over the entire sequence, alignment errors from other frames may accumulate and adversely affect its single-frame alignment quality, resulting in substantial drift. By contrast, our method preserves much better alignment at the individual-frame level.

\subsection{Quantitative Results of Camera Parameter Estimation}
As mentioned in Section~4.2 of the main text, we provide additional quantita-\begin{wraptable}{r}{0.42\textwidth}
    \vspace{-1em}
    \centering
    \small
    \setlength{\tabcolsep}{5pt}
    \begin{tabular}{lcc}
        \toprule
        Metric & KITTI & TartanAir \\
        \midrule
        ATE $\downarrow$ (m) & 0.4157 & 0.1464 \\
        RTE $\downarrow$ (m) & 1.1700   & 0.1821 \\
        RRE $\downarrow$ ($^\circ$) & 0.5711 & 1.7350 \\
        \bottomrule
    \end{tabular}
    \caption{Quantitative study of camera parameter prediction on KITTI and TartanAir.}
    \label{tab:app_camera}
    \vspace{-2em}
\end{wraptable}tive analysis of camera prediction errors on KITTI and TartanAir in terms of ATE, RTE, and RRE in Table~\ref{tab:app_camera}. The results show that VGGT-World achieves reasonable camera consistency on both datasets. These findings indicate that the learned latent dynamics preserve not only scene structure but also the underlying camera evolution without explicit pose supervision. This further supports that our geometry-centric forecasting framework is able to capture temporally coherent structure and motion cues in a unified latent space.



\section{Limitations and Future Work}\label{sec:limitations}

\paragraph{Limitations.}
Although our framework demonstrates versatility across diverse indoor and outdoor environments, several challenges remain. First, while we adopt a self-supervised, geometry-centric approach, its scalability to large, highly heterogeneous datasets is yet to be fully established, requiring further validation of long-term generalization. Second, our design intentionally avoids external supervision (e.g., ground-truth depth or camera poses) to maintain simplicity. However, this minimal conditioning limits the degree of precise control over the generation process. Extending our current framework toward fully controllable world modeling will require integrating explicit conditioning signals.

\paragraph{Future Work.}
Building on this foundation, an important next step is to scale the proposed framework to larger and more diverse datasets, which would allow a more comprehensive assessment of its scalability and generalization across complex real-world environments. In addition, incorporating extra control signals, such as pose or action, into the generation process is a promising direction for future research. Such extensions could enable more flexible and directionally controllable generation, making the framework better suited to downstream applications that require interactive prediction and planning, including autonomous driving and embodied AI. We are currently pursuing both directions in ongoing work.

\clearpage

\end{document}